\documentclass[runningheads,a4paper]{llncs}
\usepackage{wrapfig}
\usepackage{subfigure}
\usepackage{amssymb}
\usepackage{multirow}
\usepackage{graphicx}
\usepackage{floatflt}
\usepackage{url}
\urldef{\mailsa}\path|{alfred.hofmann, ursula.barth, ingrid.haas, frank.holzwarth,|
\urldef{\mailsb}\path|anna.kramer, leonie.kunz, christine.reiss, nicole.sator,|
\urldef{\mailsc}\path|erika.siebert-cole, peter.strasser, lncs}@springer.com|
\newcommand{\keywords}[1]{\par\addvspace\baselineskip
\noindent\keywordname\enspace\ignorespaces#1}

\begin{document}

\mainmatter  % start of an individual contribution

\title{Graph-based Learning\\
 with Unbalanced Clusters}

\titlerunning{Graph-based Learning with Unbalanced Clusters}

\author{Jing Qian
\and Venkatesh Saligrama \and Manqi Zhao}
\authorrunning{Graph-based Learning with Unbalanced Clusters}

\institute{Boston University, Department of Electrical and Computer Engineering,\\
8 Saint Mary's Street, Boston, MA 02215, USA} %\\
% \mailsa\\
% \mailsb\\
% \mailsc\\
% \url{http://www.springer.com/lncs}}


\toctitle{Lecture Notes in Computer Science}
\tocauthor{Authors' Instructions}
\maketitle

\begin{abstract}
Graph construction is a crucial step in spectral clustering (SC) and graph-based semi-supervised learning (SSL). Spectral methods applied on standard graphs such as full-RBF, $\epsilon$-graphs and $k$-NN graphs can lead to poor performance in the presence of proximal and unbalanced data. This is because spectral methods based on minimizing RatioCut or normalized cut on these graphs tend to put more importance on balancing cluster sizes over reducing cut values.
We propose a novel graph construction technique and show that the RatioCut solution on this new graph is able to handle proximal and unbalanced data.
Our method is based on {\it adaptively} modulating the neighborhood degrees in a $k$-NN graph, which tends to sparsify neighborhoods in low density regions. Our method adapts to data with varying levels of unbalancedness and can be naturally used for small cluster detection. We justify our ideas through limit cut analysis. Unsupervised and semi-supervised experiments on synthetic and real data sets demonstrate the superiority of our method.

\end{abstract}

%%%%%%%%%%%%%%%%%%%%%%%%%%%%%%%%%%%%%%%%%%
\section{Introduction and Motivation}\label{sec:intro_motiv}
%%%%%%%%%%%%%%%%%%%%%%%%%%%%%%%%%%%%%%%%%%
Graph-based approaches are popular tools for unsupervised clustering and semi-supervised learning(SSL). In these approaches, a graph representing the data set is first constructed. Then a graph-based learning algorithm such as spectral clustering(SC) \cite{Shi00} or SSL algorithms \cite{zhu08,WanJebCha08} is applied on the graph. %These algorithms solve the graph-cut minimization problem on the graph.
Of the two steps, graph construction has been identified to be critical\cite{zhu08,Luxburg07,Maier1,JebShc06,JebWanCha09}. Effective graph construction strategies turn out to be even more critical in the presence of unbalanced and proximal data. Unbalanced data arises routinely in many applications including multi-mode(class) clustering and SSL tasks. The focus of this paper is on graph construction for spectral methods and we refer to \cite{Fraley02} for model-based approaches. %\cite{Fraley02}, which can incorporate unbalancedness, is outside the scope of our work. Nevertheless, it is worth pointing out that while model-based approaches work well for relatively simple cluster shapes they are unable to account for complex shapes(\cite{Ng01}).

Common graph construction methods include $\epsilon$-graph, fully-connected RBF-weighted(full-RBF) graph and $k$-nearest neighbor($k$-NN) graph.
$\epsilon$-graph links two nodes $u$ and $v$ if $d(u,v)\leq \epsilon$. % $\epsilon$-graph is vulnerable to outliers due to the fixed threshold $\epsilon$.
Full-RBF graph links every pair of nodes with RBF weights $w(u,v)=exp(-d(u,v)^2/2\sigma^2)$, which is in fact a soft threshold($\sigma$ serves similarly as $\epsilon$).
$k$-NN graph links $u$ and $v$ if $v$ is among the $k$ closest neighbors of $u$ or vice versa. It is the most recommended method\cite{Luxburg07,zhu08} due to its relative robustness to outliers.
In \cite{JebShc06} the authors propose $b$-matching graph. This method is supposed to eliminate some of the spurious edges of $k$-NN graph and lead to better performance\cite{JebWanCha09}.

However, for unbalanced and proximal data clusters, SC and graph-based SSL algorithms appear to perform poorly on these conventional graphs.
This poor performance is a result of minimizing RatioCut objective on these graphs. For unbalanced and proximal data clusters the RatioCut objective on these graphs tends to put more importance on balancing cluster sizes over reducing cut values. This sometimes leads to cuts that are not meaningful. In Section~2 we will investigate the fundamental reasons that lead to poor results. We will then outline a novel graph construction strategy, whereby the RatioCut objective on this new graph is able to handle varying levels of proximal and unbalanced data. Our rank-modulated degree (RMD) graph construction method, described in detail in Section~3, is based on modulating the degrees in a $k$-NN graph. The impact of this strategy is that it results asymptotically in more edges per node in high-density regions and a sparsification near density valleys. We explore the theoretical basis for these results in Section~4. In Section~5 we present several experiments on synthetic and real datasets and show significant improvements in SC and SSL results over conventional graph constructions.

%Ratio-cut seeks to tradeoff cut-value against cut-size. While the cut-value attempts to find cuts with small weights (such as low density regions), the balancing term attempts to ensure that the size of each partition is sizable (to rule out cuts near outliers). %Although ratio-cut based solutions leads to good performance with conventional graph constructions in a number of situations, we have found that the performance can be poor for unbalanced and proximal data clusters.
%For unbalanced and proximal data clusters the ratio-cut objective tends to be biased
%
%
%these traditional graphs can lead to poor performance for SC and graph-based SSL algorithms when the data is unbalanced and proximal.
%
%-- density is shallow
%-- ratio cut objective
%--

\section{Proximal \& Unbalanced Data Clusters}
In this section we will investigate some of the reasons that lead to poor SC and SSL performance for conventional graph constructions in the presence of proximal and unbalanced data. We draw upon existing results to justify our reasoning.

Let $G=(V,E)$ be the graph constructed from $n$ samples drawn IID from some underlying density $f(x)$, where $x \in \mathbb{R}^d$. Let $(C,\bar{C})$ be a 2-partition of the nodes separated by a hyper surface $S$. The simple cut is defined as:
\begin{eqnarray}\label{equ:cut}
Cut(C,\bar{C}) = \sum_{u\in C,v\in \bar{C},(u,v)\in E}w(u,v),
\end{eqnarray}
where $w(u,v)$ is the weight of edge $(u,v) \in E$. Spectral clustering techniques are based on minimizing RatioCut:
\begin{equation}\label{equ:ratiocut}
    RatioCut(C,\bar{C})=Cut(C,\bar{C})\left(\frac{1}{|C|}+\frac{1}{|\bar{C}|}\right),
\end{equation}
where $|C|$ denotes the number of nodes in $C$. A variant of RatioCut is the so called normalized cut (NCut). Our discussions for RatioCut also extend to NCut and we will not discuss NCut from here on. Note RatioCut augments the simple Cut with a balancing term, which desensitizes partitions from outliers. \\

\noindent
{\bf Unbalanced Proximal Gaussian Mixture:}
By means of an example, we will argue that minimizing RatioCut on conventional graphs has fundamental drawbacks for clustering proximal and unbalanced datasets.
\begin{figure*}[!tb]
\begin{centering}
\begin{minipage}[t]{.45\textwidth}
\includegraphics[width = 1\textwidth]{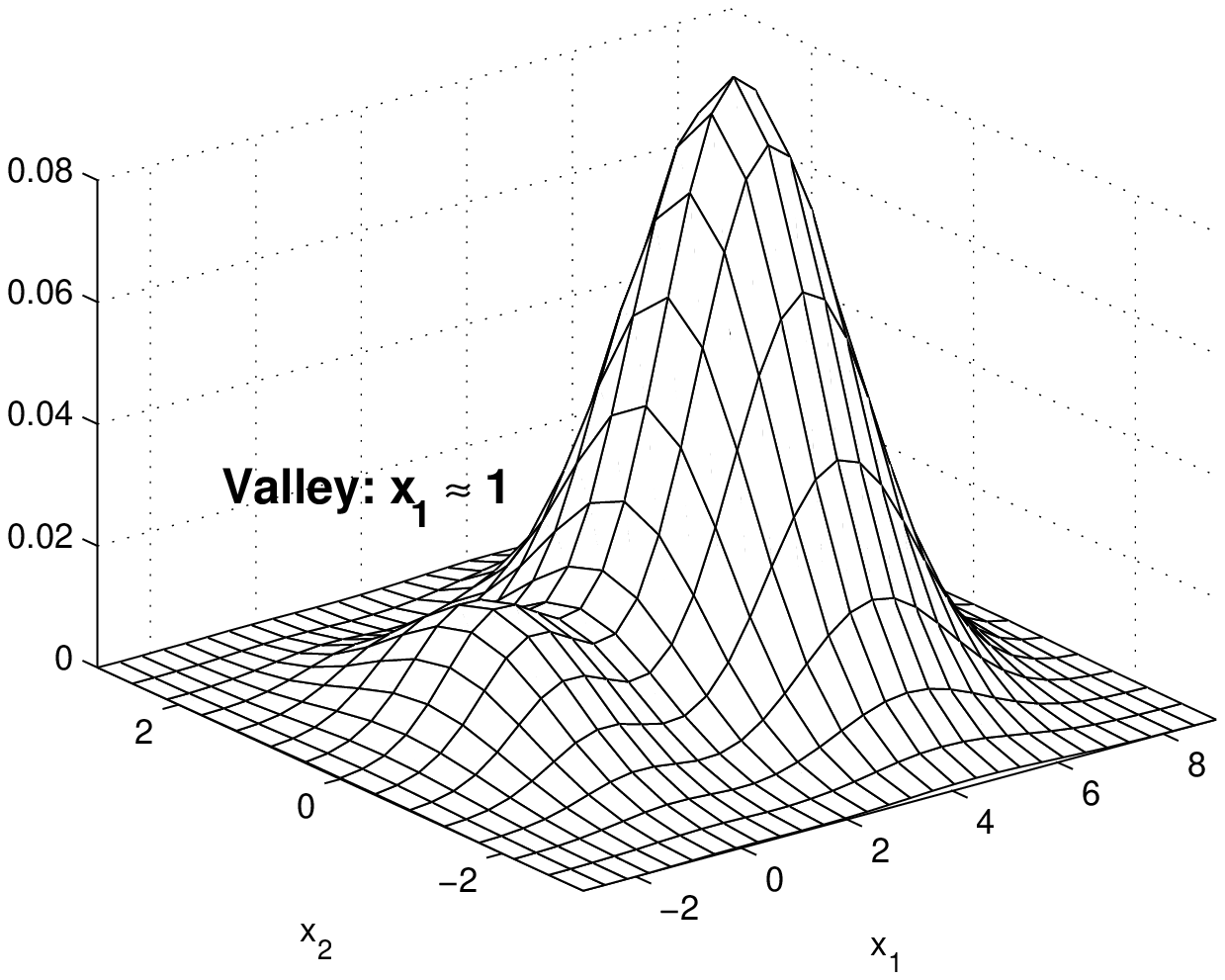}
\makebox[5.2 cm]{\small (a) pdf}
\end{minipage}
\begin{minipage}[t]{.45\textwidth}
\includegraphics[width = 1\textwidth]{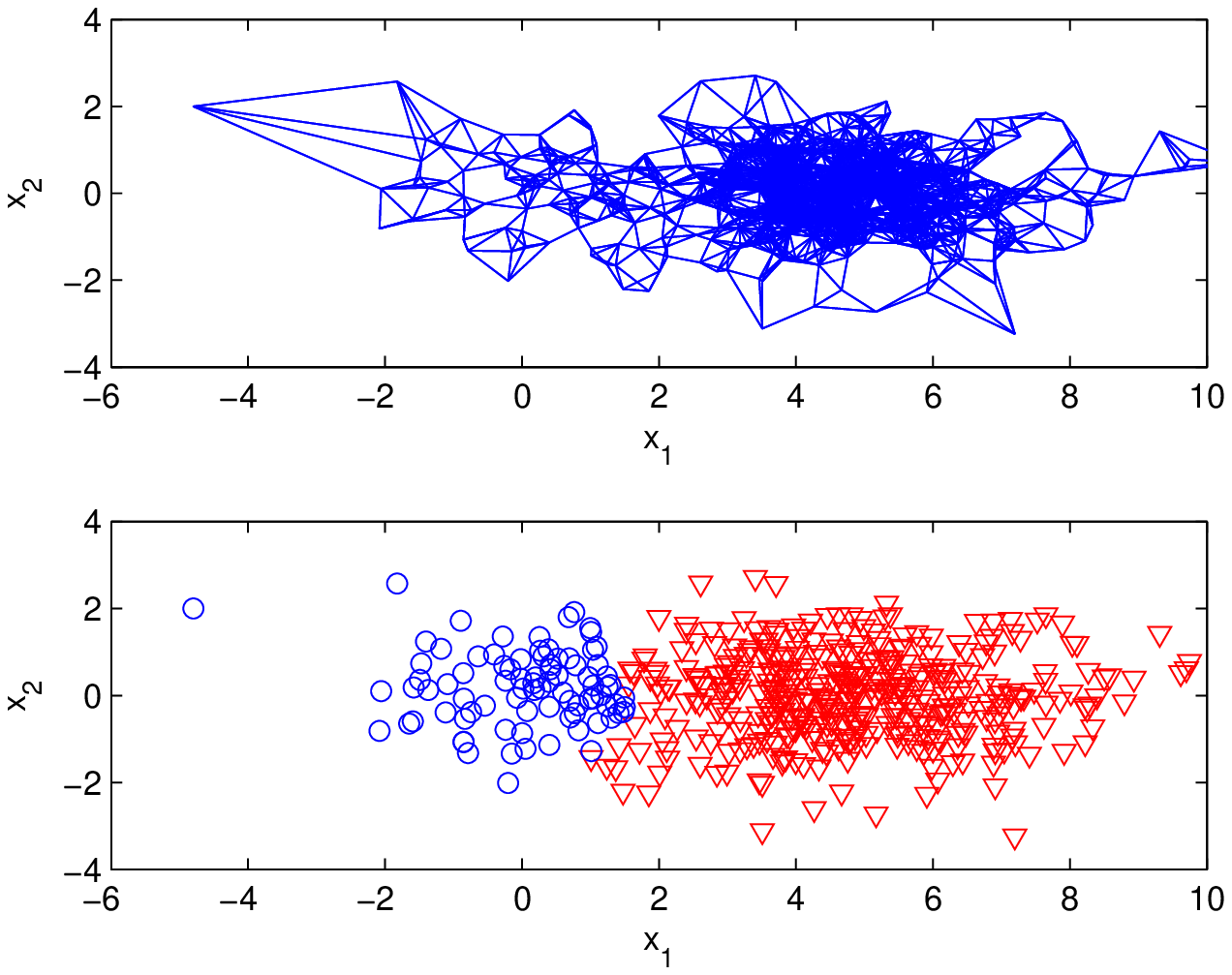}
\makebox[5.2 cm]{\small (b) result of RMD(our method) }
\end{minipage}
\begin{minipage}[t]{.45\textwidth}
\includegraphics[width = 1\textwidth]{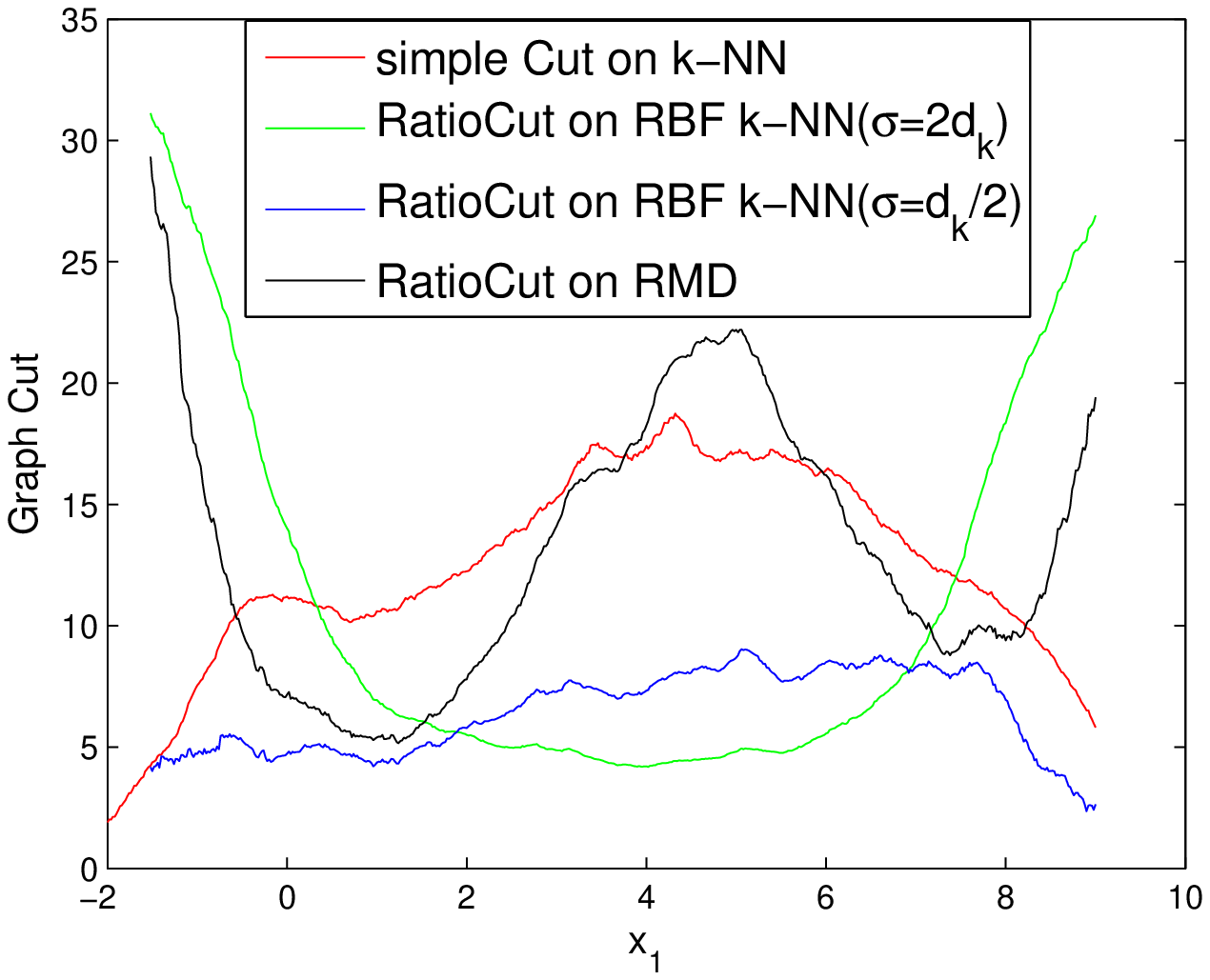}
\makebox[5.2 cm]{\small (c) Ratio Cut of $k$-NN and RMD }
\end{minipage}
\begin{minipage}[t]{.45\textwidth}
\includegraphics[width = 1\textwidth]{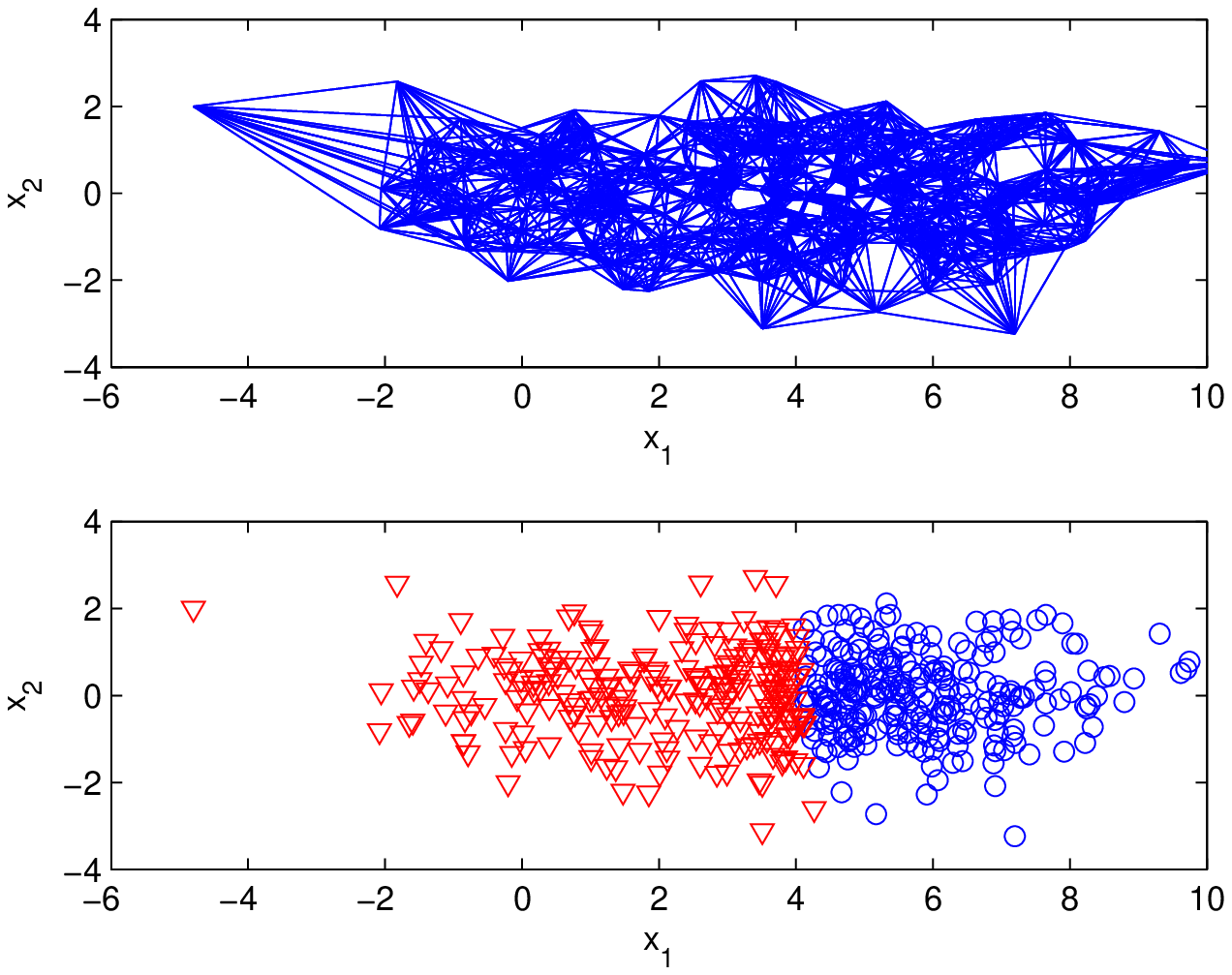}
\makebox[5.2 cm]{\small (d) result of $k$-NN}
\end{minipage}
\begin{minipage}[t]{.45\textwidth}
\includegraphics[width = 1\textwidth]{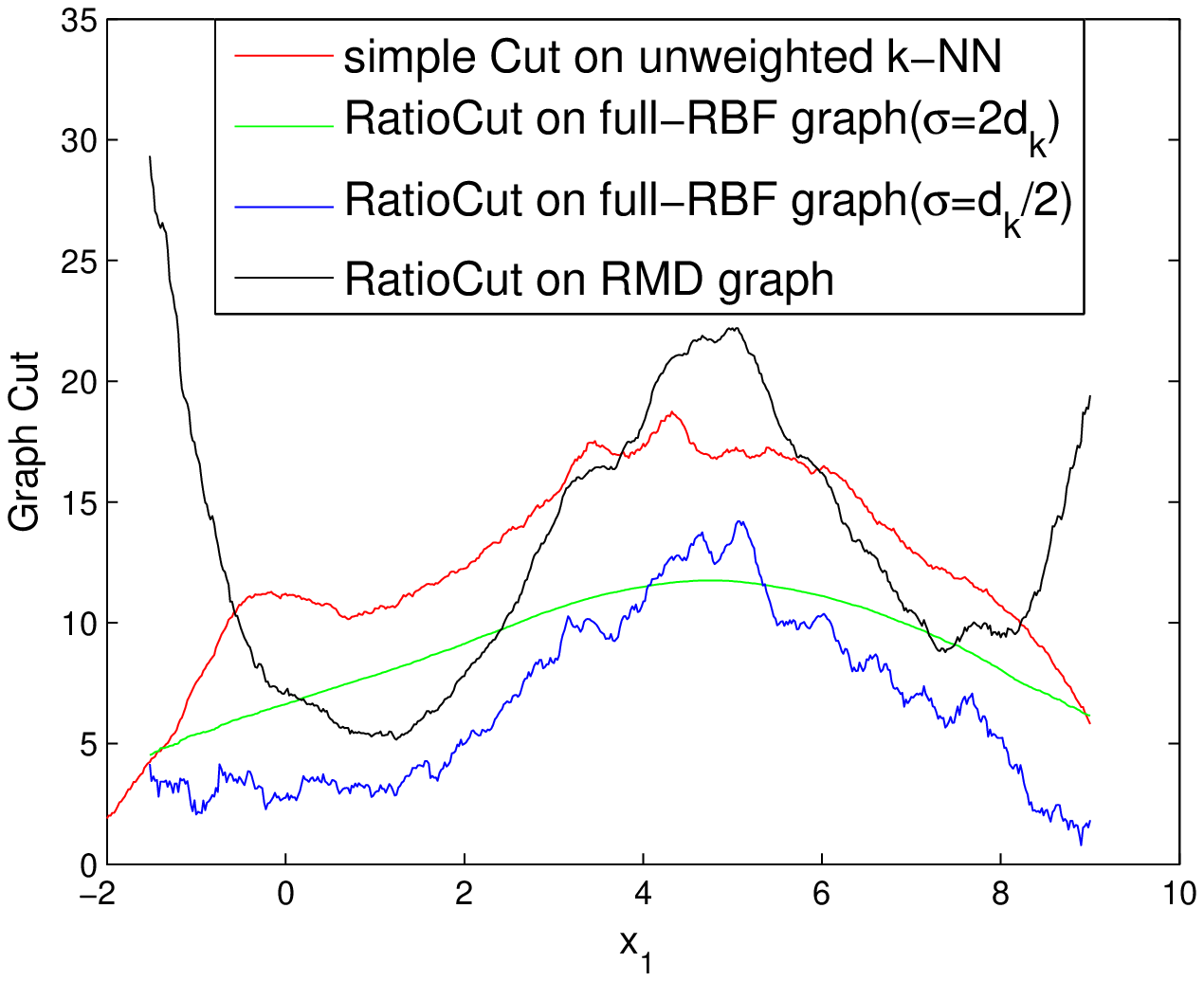}
\makebox[5.2 cm]{\small (e) Ratio Cut of full-RBF and RMD}
\end{minipage}
\begin{minipage}[t]{.45\textwidth}
\includegraphics[width = 1\textwidth]{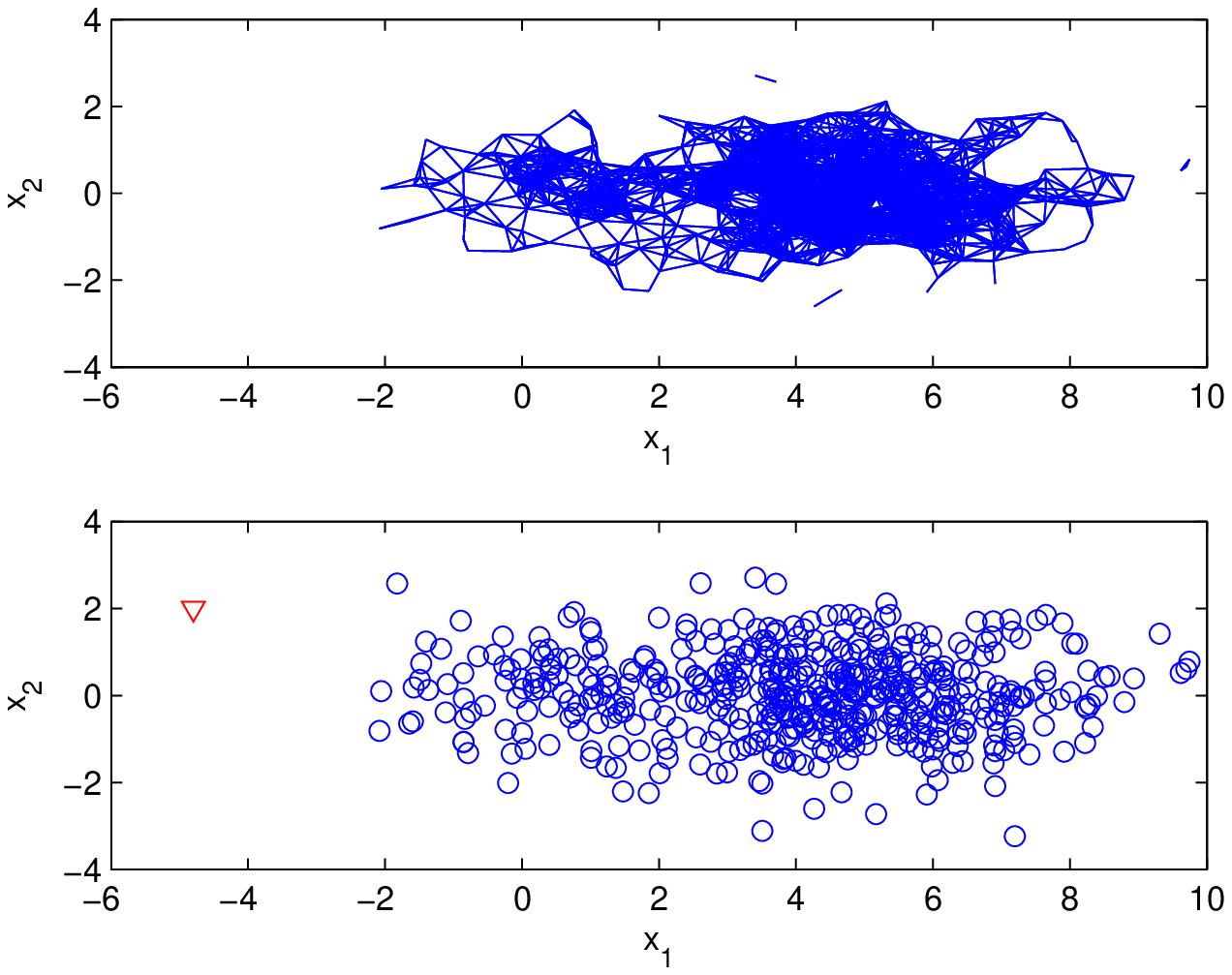}
\makebox[5.2 cm]{\small (f) result of full-RBF($\epsilon$-graph)}
\end{minipage}
\caption{\small Various graphs and SC results. Cut and RatioCut values of (c),(e) are averaged over 20 Monte Carlo runs. The values are re-scaled here for demonstration. $\bar{d}_k$ is the average $k$-NN distance. $n=1000$, $k=30$. For (b) unweighted RMD graph with $l=30, \lambda=0.4$; for (d) unweighted $k$-NN; for (f) $\epsilon=\sigma=\bar{d}_{k}$ is used.}
\label{fig:2g_graph}
\end{centering}
\vspace*{-0.2in}
\end{figure*}
For our illustrative experiment we consider $n=1000$ data samples drawn IID from a proximal and unbalanced 2-D gaussian mixture density,
\begin{equation} \label{e.mixgauss}
f(x) \triangleq \sum^{2}_{i=1}\alpha_iN(\mu_i,\Sigma_i)
\end{equation}
where $\alpha_1$=0.9, $\alpha_2$=0.1, $\mu_1$=[4.5;0], $\mu_2$=[0;0], $\Sigma_1=diag(2,1), \Sigma_2=I$, as shown in Fig.\ref{fig:2g_graph}. We examine different graph constructions including full-RBF, (RBF) $k$-NN and $\epsilon$-graph. Note that these graph constructions are parameterized by $k$, $\sigma$, $\epsilon$. Our SC results here are depicted for reasonable choices of these parameters.

A balanced cut in this case is approximately a line parallel to $x_2$ axis passing through $x_1 = 4$. A cut at the density valley is approximately a line parallel to $x_2$ axis passing through $x_1=1$. For the SC to seek a cut at the valley we would need the RatioCut to achieve its minimum at $x_1 = 1$.

The re-scaled simple Cut curve in (c),(e) shows that the Cut value is relatively large at $x_1=1$ due the fact that the density valley is "shallow."  Fig.1(c) shows RatioCut values for RBF $k$-NN for large and small $\sigma$ values. Large $\sigma$ (unweighted $k$-NN behaves similarly) achieves minimum at the balanced position ($x_1\approx 4$); while small $\sigma$ pulls down RatioCut near the boundaries and turns out to be vulnerable to outliers. (e) shows fhat full-RBF($\epsilon$-graph behaves similarly) with large $\sigma$ tends to smooth out the curve and is insensitive to location of the valley, while small $\sigma$ appears to be vulnerable to outliers. In contrast our method, RMD, appears to be able to reject outliers and achieves minimum RatioCut close to the valley position. \\

%Fig.\ref{fig:2g_graph}(c) shows the RatioCut curve of our method is significantly decreased near density valley, while remains large near boundaries. Therefore our method is able to capture the density valley and robust to outliers as well, as demonstrated in (b).
%Note that by varying $\alpha_i$ we can vary the extent of unbalancedness. Proximity can be varied by varying $\mu_i,\,\sigma_i$.

\noindent
{\bf Graph Partitioning, Cut-values, and Cluster Sizes:}
%To gain further insight into this example we examine cut value expressions for varying unbalanced cluster sizes. %Although, our analysis here is general and extends to other unbalanced data clusters we consider the mixture Gaussian example of Eq.~\ref{e.mixgauss} for simplicity of exposition.
By varying $\alpha_i$ in Eq.(\ref{e.mixgauss}) we can vary the size of unbalanced clusters; varying $\mu_i,\sigma_i$ has the effect of varying proximity of the clusters. For a given value of $\alpha_i, \mu_i, \sigma_i$, we let $S_U$ be the locus of points corresponding to the density valley (for example in Fig.1 this is the line $x_1=1$), and $S_B$ any line that asymptotically results in two balanced partitions (for example in Fig.1 this is the line $x_1=4$).
Now for a graph $G=(V,E)$, the lines $S_U$ and $S_B$ describe two different partitions, one unbalanced but respecting the inherent clustering of data and the other balanced but not respecting the underlying data clusters.
We denote by $C_U, \bar C_U$ the partitions resulting from a cut associated with the line $S_U$ and by $C_B,\,\bar C_B$ the partitions resulting from a cut associated with the line $S_B$ \footnote{data samples situated exactly on the line $S_U$ or $S_B$ are randomly assigned.}. The Cut-ratio $q$ is defined as the ratio of the Cut values corresponding to the two partitions; $y$ denotes the size of unbalanced partition, namely,
\begin{equation} \label{e.cutratio}
q= {Cut(C_U,\bar C_U) \over Cut(C_B,\bar C_B)},\,\,\,\,\, y = {1 \over n} \min \{|C_U|,\,|\bar C_U|\}
\end{equation}
Now we examine the condition when the natural unbalanced partition has a smaller RatioCut value than the balanced partition. This requires that,
\begin{equation} \label{e.limits}
    Cut(C_U,\bar C_U)(\frac{1}{y n}+\frac{1}{(1-y)n})<Cut(C_B,\bar C_B) (\frac{1}{n/2}+\frac{1}{n/2}) \Longrightarrow q < 4y \left(1-y \right)
\end{equation}
where we have substituted for $q$ from Eq.(\ref{e.cutratio}). A plot of the Cut-ratio $q$ for different unbalanced proportions $y$ is shown in Fig.2.

\begin{wrapfigure}{r}{0.45\textwidth}
%\begin{figure}[tb]
\vspace{-15pt}
\centering
\includegraphics[height=.4\textwidth]{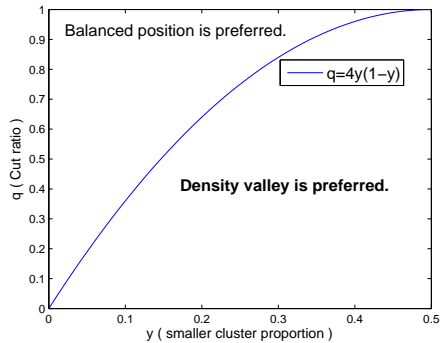}
\caption{\small Cut-ratio ($q$) vs unbalanced cluster size ($y$). Ratio Cut value is smaller for balanced cuts over natural unbalanced cuts whenever the cut-ratio is above the curve.}
\label{fig:qy}
\end{wrapfigure}
\emph{Consequently, Fig.~2 and Eq.(\ref{e.limits}) points to a fundamental aspect of RatioCut for datasets with unbalanced and proximal clusters.} If the tuple $(q,y)$ lies above the curve, RatioCut value is smaller for balanced partitions than partitioning at density valley (note $y \approx 0.1$ required Cut-ratio can be as small as $0.36$). \\
%balanced cuts are preferred over natural unbalanced ones even for relatively small cut-ratios ($q \approx 0.25$).

\noindent
{\bf Why do conventional graphs fail?}
%With this insight we now examine why conventional graph construction methods perform poorly for unbalanced proximal clusters.
This is best explained by understanding the limit-cut analysis results for $k$-NN, $\epsilon$-graph and full-RBF graphs \cite{Maier1,Narayanan06}. For appropriately chosen parameters, $k$, $\sigma$ and $\epsilon$ respectively, as the number of samples $n\rightarrow \infty$, the Cut ratio $q$ and the unbalanced cluster size $y$ converges (with high probability) to:
\begin{equation} \label{e.limval}
q \stackrel{n \rightarrow \infty}{\longrightarrow} {\int_{S_U} f^{\gamma}(x)dx \over \int_{S_B} f^{\gamma}(x)dx},\,\,\,\,\, y \stackrel{n \rightarrow \infty}{\longrightarrow} \min \{\mu(C_U),\,\mu(\bar C_U)\}
\end{equation}
where $\mu(C_U),\,\mu(\bar C_U)$ are the volumes (probability) of sets $C_U$ and $\bar C_U$ under density $f(x)$ respectively. $\gamma$ is a constant and depends on the specific graph construction.
While standard graph construction methods do account for the underlying density $f(x)$, this by itself is insufficient for proximal and unbalanced clusters. For the mixture Gaussian case (Eq.(\ref{e.mixgauss})) it follows from Eq.(\ref{e.limval}) that $q$ can be relatively large for an appropriate choice of $\mu_i,\,\sigma_i$ and a fixed choice of unbalancedness, $y$. Note, $y$, is predominantly controlled through mixture proportions $\alpha_i$. Eq.(\ref{e.limits}) and Fig.~2 asserts that in this case RatioCut has a smaller value for balanced partitions even when density valley cut, $S_U$, is the natural choice.

\noindent
{\bf Parameter tuning:}
It is possible that the parameters $k$, $\sigma$, and $\epsilon$ can be tuned to account for unbalancedness. However, large values of $k$, $\sigma$ and $\epsilon$ tends to smooth the underlying distribution (see Fig.~1) and increases the Cut-ratio, which worsens the problem. In contrast decreasing $k$, $\sigma$ and $\epsilon$ below well-understood acceptable thresholds (see \cite{Maier1,Narayanan06}) leads to disconnected graphs and sensitivity to outliers (this is also seen in Fig.~1). While changing parameters $k,\sigma,\epsilon$ can globally modify the graph topology, this has poor control over Cut-ratio. For instance, increasing/decreasing $k$ results in a $k$-NN graph with uniformly larger/smaller number of neighbors for all the nodes and uniformly larger/smaller Cut values for any cut, leading to poor control of Cut-ratio.

\noindent
{\bf Controlling Cut Ratio through Graph Sparsification:}
From the above discussion it is clear that we need to directly control Cut-ratio. We do so by adaptively sparsifying graph neighborhoods.  Neighborhoods for nodes in plausible low-density regions are sparsified and those in high-density regions are ``densified''. By controlling this sparsification/densification the Cut-ratio is controlled and adapted to varying degrees of unbalancedness and proximity. Comparisons between standard constructions and our RMD graph for the Gaussian mixture of Eq.(\ref{e.mixgauss}) are shown in Fig.~1. As seen our method sparsifies low density regions in contrast to other methods.

%%%%%%%%%%%%%%%%%%%%%%%%%%%%%%%%%%%%%%%%%%
\section{RMD Graphs: Main Steps}\label{sec:RMD_idea}
%%%%%%%%%%%%%%%%%%%%%%%%%%%%%%%%%%%%%%%%%%

%We propose to modulate the node degrees through a parameterized scheme, which is based on ranking of data samples. We call the resulting graph the Rank-Modulated Degree(RMD) graph. Our graph is able to find the valley cut while being robust to outliers. Moreover we provide a model selection step making our graph adaptable to data with varying levels of unbalancedness.
Given data samples $\{x_1,\ldots,x_n\}$ in $\mathbb{R}^d$, our rank-modulated degree(RMD) graph based learning involves the following steps:

\noindent
{\bf (1) Rank Computation:}
The rank $R(x)$ of every point $x$ is calculated:
\begin{eqnarray}\label{eq:grank}
  R(x) = \frac{1}{n}\sum_{i=1}^n\mathbb{I}_{\{G(x)\leq G(x_i)\}}
\end{eqnarray}
where $\mathbb{I}$ denotes the indicator function. Ideally we would like to choose $G(\cdot)$ to be the underlying density, $f(\cdot)$ of the data. Since $f$ is unknown, we need to employ some surrogate statistic. While many choices are possible, the statistic in this paper is based on $k$-nearest neighbor distances. Such rank based statistics have been employed for high-dimensional anomaly detection~\cite{zhaonips,zhaoaistats12}. More choices for $G$ and a robust procedure for computing $R(x)$ are described in Sec.\ref{subsec:rank}. The rank is a normalized ordering of all points based on $G$, ranges in $[0,1]$, and indicates how extreme the sample point $x$ is among all the points.

\noindent
{\bf (2) RMD Graph Construction:}
Connect each point $x$ to its deg($x$) closest neighbors. The number of neighbors deg($x$) for point $x$ is modulated as follows:
\begin{eqnarray}\label{eq:degree}
  deg(x) = k(\lambda+2(1-\lambda)R(x))
\end{eqnarray}
where, $\lambda$ is a scalar parameter that will be optimized later.
Here $k$ is the average degree, $\lambda\in [0,\,1]$ controls the minimum degree. It is not difficult to see that $R(x)$ converges (in distribution) to a uniform measure on the unit interval regardless of the underlying density $f(\cdot)$ if $G(\cdot)$ is bijective. This implies that the expected value converges to 0.5. Consequently, the average degree across all samples is $k$. Furthermore, the above modulation scheme can be thought of as modulating the degree of each node around a nominal value equal to $k$. The remaining issue is to optimize over the scalar parameter $\lambda$, which is described in Step (4). %We choose $\lambda$ through an optimization step (see below Step (4)).

\noindent
{\bf (3) Graph-based Learning:}
The third step involves using RMD graph in a graph-based clustering or SSL algorithm. Spectral clustering algorithms based on RatioCut for 2-class and multi-class clustering are now well established. For SSL algorithms we employ Gaussian Random Fields(GRF) and Graph Transduction via Alternating Minimization(GTAM). These approaches all involve minimizing $Tr(F^TLF)$ plus some constraints or penalties, where $F$ is the cluster indicator function or classification (labeling) function, $L$ is the graph Laplacian matrix. This has been shown to be equivalent to minimizing RatioCut(NCut) for unnormalized(normalized) $L$ \cite{Chung96}. We refer readers to references \cite{Luxburg07,zhu08,WanJebCha08} for details.

\noindent
{\bf (4) Optimization over $\lambda$:}
Our final step is to optimize over $\lambda \in [0,\,1]$. Our main assumption is that we have prior knowledge that the smallest cluster is at least of size $\delta n$. We consider the 2-cluster case first. The 2-partitions resulting from spectral clustering algorithms are now parameterized by $\lambda$: $\left(C(\lambda),\bar{C}(\lambda)\right)$. We now optimize the minimum Cut value over all admissible $\lambda$ such that the smallest cluster is no smaller than some threshold $\delta$:
\begin{eqnarray}\label{eq:selection}
  & J(\delta)=\min_{\lambda\in [0,1]}\{Cut(C(\lambda),\bar{C}(\lambda)\} \\
\nonumber
  & s.t. ~~\min\{|C(\lambda)|,|\bar{C}(\lambda)|\}\geq \delta n
\end{eqnarray}
%Here we have formulated the problem for the 2-partition case but straightforward extensions hold for multi-cluster and SSL. %We use the case of 2-partition clustering for illustrative purposes and for brevity. For this case we can choose a $\lambda$ to minimize the simple Cut value over all sizable partitions. %Details of this step are described in Sec.~\ref{subsec:cv_scheme}.
%Specifically, let $\Lambda$ be the set of choices for $\lambda$. The 2-partitions resulting from spectral clustering algorithms are now parameterized by $\lambda$: $\left(C(\lambda),\bar{C}(\lambda)\right)$. We now look for the minimum Cut value over all admissible $\lambda$ such that the smallest cluster is no smaller than some threshold $\delta$:
%\begin{eqnarray}\label{eq:selection}
%  & J(\delta)=\min_{\lambda\in\Lambda}\{Cut(C(\lambda),\bar{C}(\lambda))\} \\
%\nonumber
%  & s.t. ~~\min\{|C(\lambda)|,|\bar{C}(\lambda)|\}\geq \delta n
%\end{eqnarray}
$\delta$ sets the threshold of minimum cluster size, which means clusters of smaller sizes than $\delta n$ are viewed as outliers and will be discarded. Algorithms for K-partition clusters and SSL algorithms can be extended in a similar manner by optimizing suitable objective functions in place of the 2-partition cut value.
Note that a similar optimization step can also be applied to select the best $k$ and $\sigma$ with traditional graph constructions as well. We will employ this strategy for the purpose of comparison on real data sets in Sec.\ref{subsec:real}.

%\begin{floatingfigure}[r]{0.5\textwidth}

%%%%%%%%%%%%%%%%%%%%%%%%%%%%%%%%%%%%%%%%%%
\subsection{Rank Computation}\label{subsec:rank}
%%%%%%%%%%%%%%%%%%%%%%%%%%%%%%%%%%%%%%%%%%
The missing component in our RMD method is the specification of the statistic $G$. We choose the statistic $G$ in Eq.(\ref{eq:grank}) based on nearest-neighbor distances. Specifically,
\begin{equation}\label{equ:G(x)}
G(x)=\frac{1}{l}\sum^{l+\lfloor \frac{l}{2} \rfloor}_{i=l-\lfloor \frac{l-1}{2} \rfloor}D_{(i)}(x)
\end{equation}
where $D_{(i)}(x)$ denotes the distance from $x$ to its $i$-th nearest neighbor, and $G$ is the average of $x$'s $\frac{l}{2}$-th to $\frac{3l}{2}$-th nearest neighbor distances. Other choices for $G$ are listed below.

\noindent
(1) $\epsilon$-Neighborhood: $G(x)$ is the number of neighbors within an $\epsilon$-ball of $x$.

\noindent
(2) $l$-Nearest Neighorhood: $G(x)$ is the distance from $x$ to its $l$-th nearest neighbor.

Empirically (and theoretically) we have observed that the average nearest neighbor distance leads to better performance and robustness. To reduce variance during rank computation we adopt a U-statistic resampling technique \cite{Korolyuk94} with $B$ resamplings.

\noindent
\textbf{U-statistic Resampling For Rank Computation:}\\
Given $N=2m$ data points, \\
(a) Randomly split the data into two equal parts: $S_1=\{x_1,...,x_m\}$, $S_2=\{x_{m+1},...,x_{2m}\}$. \\
(b) Points in $S_2$ are used to calculate $G$ for $x_i \in S_1$ according to Eq.(\ref{equ:G(x)}), and vice versa. \\
(c) Ranks of $x_i\in S_1$ are computed by Eq.(\ref{eq:grank}) within $S_1$ and similarly for $x_i\in S_2$. \\
(d) Resplit the data and repeat the above steps $B$ times. Let $R_b(x_i)$ be the rank of $x_i$ obtained from the $b$-th resampling. We then use the average as the final rank:
\begin{equation}
    R(x_i)=\frac{1}{B}\sum_{b=1}^{B}R_b(x_i),~~~~i=1,2,\ldots,N
\end{equation}

%%%%%%%%%%%%%%%%%%%%%%%%%%%%%%%%%%%%%%%%%%

%\begin{figure}[tb]
%\begin{floatingfigure}[r]{0.5\textwidth}
%\centering
%\includegraphics[height=5cm]{gauss2D_pdf_rank.eps}
%\caption{\small Density level sets and rank estimates for unbalanced and proximal gaussian mixture density. High/low ranks correspond to high/low density levels.}
%\label{fig:rwcont}
%\end{floatingfigure}

\noindent
{\bf Properties of the Ranked Data:}

\noindent
{\bf (1) High/Low Density Indicator:} The value of $R(x)$ is a direct indicator of whether $x$ lies in high/low density regions(Fig.\ref{fig:rwcont}).

\noindent
{\bf (2) Smoothness:}
$R(x)$ is the integral of pdf asymptotically(see Thm.\ref{rank-pvalue} in Sec.\ref{sec:thm}). It's smooth and uniformly distributed in $[0,1]$. This makes it appropriate to modulate the degrees with control of minimum, maximal and average degree.

\noindent
{\bf (3) Precision:} We do not need our estimates to be precise for every point; the resulting cuts will typically depend on relatively low ranks rather than the exact value, of most nearby points.
%\end{itemize}

\begin{wrapfigure}{r}{0.45\textwidth}
\vspace{-20pt}
\centering %\begin{center}
\includegraphics[width=0.4\textwidth]{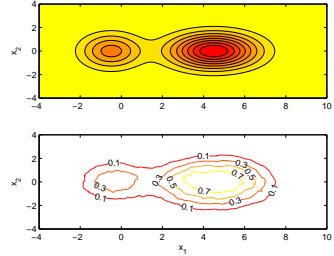}
%\end{center}
\caption{\small Density level sets \& rank estimates for unbalanced and proximal gaussian mixtures. High/low ranks correspond to high/low density levels.}
\label{fig:rwcont}
\end{wrapfigure}

%%%%%%%%%%%%%%%%%%%%%%%%%%%%%%%%%%%%%%%%%%
\subsection{Salient Properties of RMD Graphs}\label{subsec:degree}
%%%%%%%%%%%%%%%%%%%%%%%%%%%%%%%%%%%%%%%%%%

Our scheme successfully solves the following issues:

\noindent
{\bf (1) Captures density valley: }
The monotonicity of deg($x$) in $R(x)$ immediately implies that nodes in low/high density areas will have fewer/more edges, thus reducing cut-ratio $q$ in Fig.\ref{fig:qy} and ensuring that the RatioCut has low values at density valleys.

\noindent
{\bf (2) Robustifies against Outliers: }
The minimum degree of nodes in RMD graph is $k\lambda$, even for distant outliers. Furthermore, $\lambda$ is the solution to the optimization step (see Eq.~\ref{eq:selection}), and so is robust to outliers as shown in Fig.\ref{fig:2g_graph}(c), where the RatioCut curve of RMD graph(black) goes up near boundaries, guaranteeing the valley minimum is the global minimum.

\noindent
{\bf (3) Adapts to Unbalanced Clusters: }
The optimization problem of Eq.(\ref{eq:selection}) leads to sizable clusters that can be unbalanced. The reason is that small values of $\lambda$ emphasize the Cut value over the balancing term. This has the effect of preferring smaller Cut values with possibly unbalanced partitions over balanced partitions with larger Cut values. This effect is magnified because smaller $\lambda$ leads to sparser connections at low-density areas. Since the balancing term is not impacted, varying $\lambda$ from 1 to 0 moves the partition from the relatively balanced position toward the density valley (see also Thm.\ref{part2} in Sec.\ref{sec:thm}). Practically, $\lambda$ provides a flexibility to optimize the tradeoff between the simple Cut and the cluster size. The cluster-size threshold $\delta$ in the optimization step (Eq.(\ref{eq:selection})) is used to constrain clusters that are not too small, thus avoiding outliers. We can also iterate over $\delta$ to find possibly different valley cuts of different sizes. This procedure can sometimes be used for size-constrained clustering~\cite{Hoppner08}. We will demonstrate some of these ideas in Sec.5.3.

%%%%%%%%%%%%%%%%%%%%%%%%%%%%%%%%%%%%%%%%%%
\section{Analysis}\label{sec:thm}
%%%%%%%%%%%%%%%%%%%%%%%%%%%%%%%%%%%%%%%%%%

The proofs of theorems here appear in the Appendix section.
Assume the data set $\{x_1,\ldots,x_n\}$  is drawn i.i.d. from density $f$ in $\mathbb{R}^d$. $f$ has a compact support $C$. Let $G=(V,E)$ be the RMD graph. Given a separating hyperplane $S$, denote $C^+$,$C^-$ as two subsets of $C$ split by $S$, $\eta_d$ the volume of unit ball in $\mathbb{R}^d$.

First we show the asymptotic consistency of the rank $R(y)$ of some point $y$. The limit of $R(y)$, $p(y)$, is the complement of the volume of the level set containing $y$. Note that $p$ exactly follows the shape of $f$, and always ranges in $[0,1]$ no matter how $f$ scales.

\begin{theorem}\label{rank-pvalue}
Assume the density $f$ satisfies some regularity conditions. For a proper choice of parameters of $G$, as $n\rightarrow\infty$, we have
\begin{equation}
    R(y)\rightarrow p(y):= \int_{\left\{x:f(x)\leq
f(y)\right\}}f(x)dx.
\end{equation}
%, where the p-value term $p(\cdot)$ is defined as
%\begin{eqnarray}\label{def:pvalue}
%p(x)= \int_{\left\{y:f(y)\leq
%f(x)\right\}}f(y)\text{d}y
%\end{eqnarray}
\end{theorem}

Next we study RatioCut induced on unweighted RMD graph(similar for NCut). The limit cut expression on RMD graph involves an additional adjustable term which varies according to the density. This implies the Cut values in high density areas can be significantly more expensive than in low density areas. Notice that this effect becomes stronger when $\lambda$ varies from 1 to 0, which means the minimum will be attained at even smaller density areas. For technical simplicity, we assume RMD graph ideally connects each point $x$ to its deg$(x)$ closest neighbors.

\begin{theorem}\label{part2}
Assume the smoothness assumptions in \cite{Maier1} hold for the density $f$, and $S$ is a fixed hyperplane in $\mathbb{R}^d$. For unweighted RMD graph, set the degrees of points according to Eq.(\ref{eq:degree}), where $\lambda\in(0,1)$ is a constant. Let $\rho(x)=\lambda+2(1-\lambda)p(x)$. Assume $k_n/n\rightarrow{0}$. In case $d$=1, assume $k_n/\sqrt{n}\rightarrow\infty$; in case $d\geq$2 assume $k_n/\log{n}\rightarrow\infty$. Then as $n\rightarrow\infty$ we have that:
\begin{equation} \label{eq:rwncut}
    \frac{1}{k_n}\sqrt[d]{\frac{n}{k_n}}RatioCut_n(S)\longrightarrow  C_d\int_S{f^{1-\frac{1}{d}}(s)\rho(s)^{1+\frac{1}{d}}ds}\left(\mu(C^+)^{-1}+\mu(C^-)^{-1}\right).
\end{equation}
where $C_d = \frac{2\eta_{d-1}}{(d+1)\eta_d^{1+1/d}}$, $\mu(C^{\pm})=\int_{C^{\pm}}f(x)dx$.
\end{theorem}

Compared to the limit expression on $k$-NN graph(\cite{Maier1}), there is an additional term $\rho(x)=(\lambda+ 2(1-\lambda)p(x))$ here. To see the impact suppose $\lambda$ is small; we see that for $S$ near modes, $p(x)\approx 1$ and this extra term is nearly $(2)^{1+\frac{1}{d}}$. For $S$ passing valleys this term is nearly $(\lambda)^{1+\frac{1}{d}}<1$. So graph-cut value near modes are penalized more than valleys.

%%%%%%%%%%%%%%%%%%%%%%%%%%%%%%%%%%%%%%%%%%
\section{Simulations}\label{sec:experiment}
%%%%%%%%%%%%%%%%%%%%%%%%%%%%%%%%%%%%%%%%%%
Many of the examples in this section focus on the unbalanced datasets. Unbalanced data is obtained by sampling the data set in an unbalanced way. Some general simulation parameters are:\\
{\bf (1)} In U-statistic rank calculation (Sec.\ref{subsec:rank}), we fix the resampling time $B=5$.\\
{\bf (2)} All error rate results are averaged over 20 trials.\\
Other parameters will be specified below.

%%%%%%%%%%%%%%%%%%%%%%%%%%%%%%%%%%%%%%%%%%
\subsection{Multi-Cluster Complex-Shaped Clusters}\label{subsec:syn}
%%%%%%%%%%%%%%%%%%%%%%%%%%%%%%%%%%%%%%%%%%
\begin{figure*}[tb]
\begin{centering}
\begin{minipage}[t]{.46\textwidth}
\includegraphics[width = 1\textwidth]{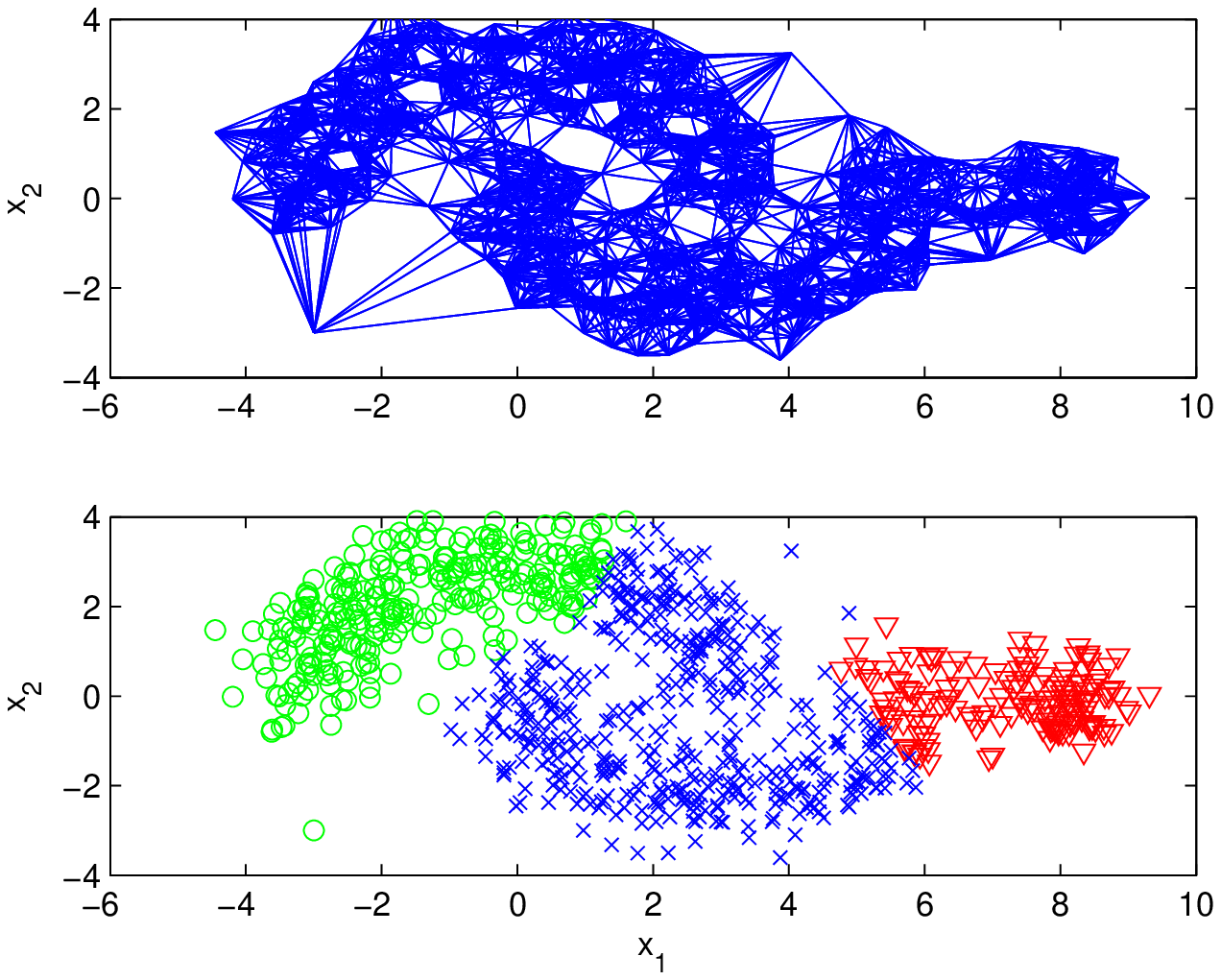}
\makebox[5.5cm]{\small (a) $k$-NN}
\end{minipage}
\begin{minipage}[t]{.46\textwidth}
\includegraphics[width = 1\textwidth]{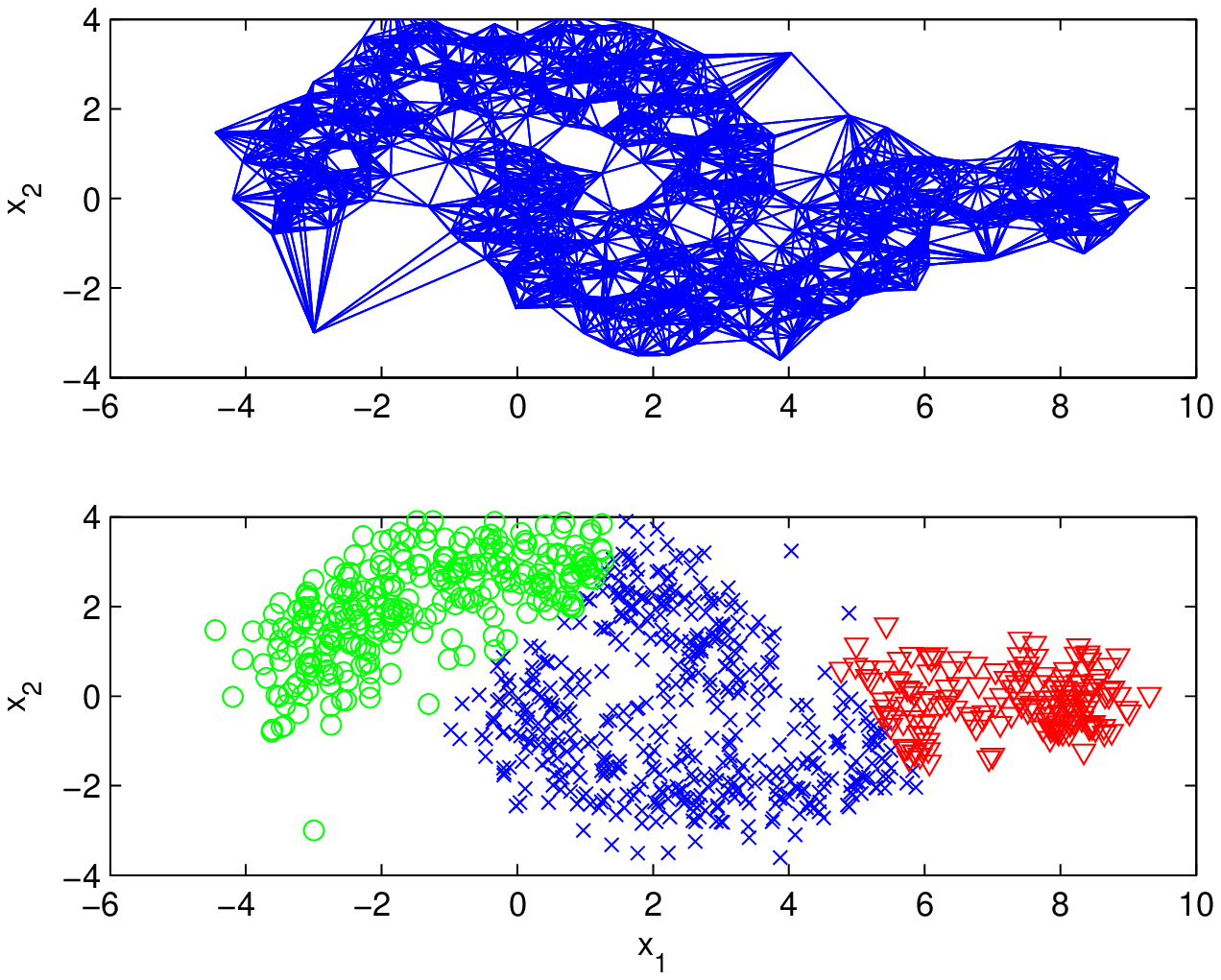}
\makebox[5.5cm]{\small (b) $b$-matching}
\end{minipage}
\begin{minipage}[t]{.46\textwidth}
\includegraphics[width = 1\textwidth]{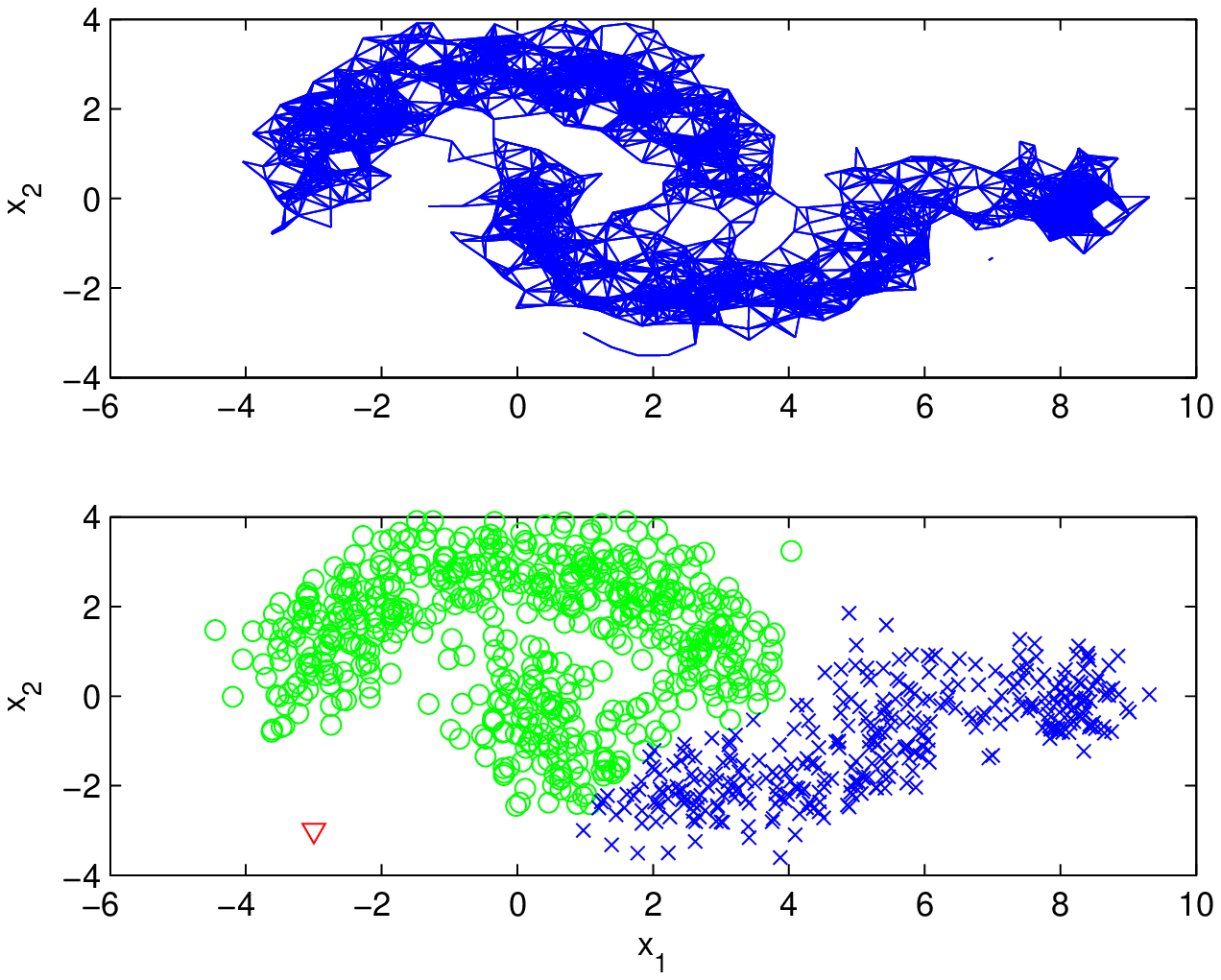}
\makebox[5.5cm]{\small (c) $\epsilon$-graph(full-RBF)}
\end{minipage}
\begin{minipage}[t]{.46\textwidth}
\includegraphics[width = 1\textwidth]{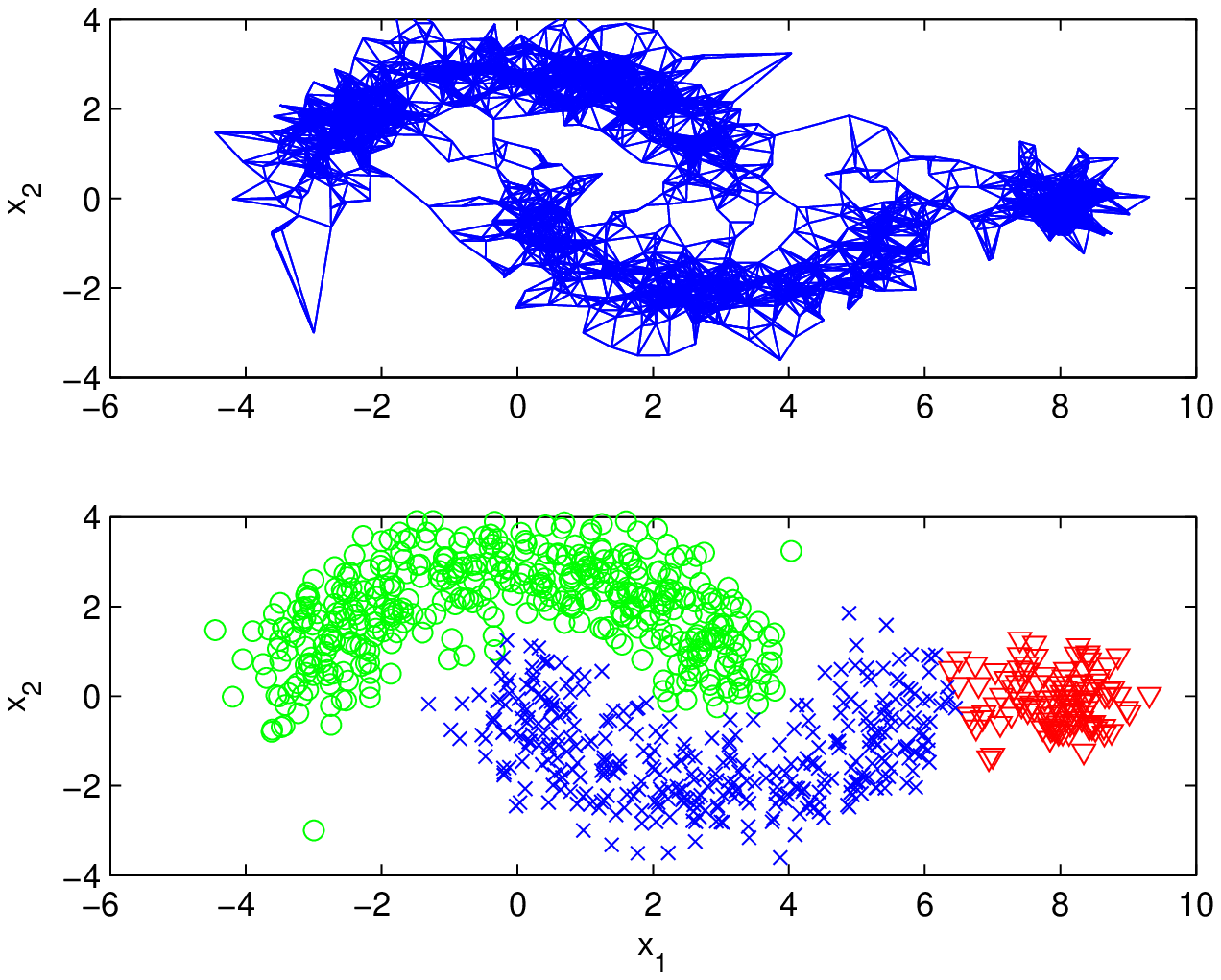}
\makebox[5.5cm]{\small (d) RMD}
\end{minipage}
\caption{\small Graphs and clustering results of SC on 2 moons and 1 gaussian data set. SC on full-RBF($\epsilon$-graph) completely fails due to the outlier. For $k$-NN and $b$-matching graphs SC cannot recognize the long winding low-density regions between 2 moons, and fails to find the rightmost small cluster. Our method significantly sparsifies the graph at low-density regions, enabling SC to cut along the curved valley, detect the small cluster and is robust to outliers as well.}
\label{fig:complex_shape}
\end{centering}
\end{figure*}
%\textbf{Multi-Cluster Complex-Shaped Example:}
Consider a data set composed of 1 small Gaussian and 2 moon-shaped proximal clusters shown in Fig.\ref{fig:complex_shape}. Sample size $n=1000$ with the rightmost small cluster $10\%$ and two moons $45\%$ each. In this example, for the purpose of illustration, we did not optimize $\lambda$ or any of the other parameters. We fix $\lambda = 0.5$, and choose $k=l=30$, $\epsilon=\sigma=\tilde{d}_k$, where $\tilde{d}_k$ is the average $k$-NN distance. On $k$-NN and $b$-matching graphs SC fails for two reasons: (1) SC cuts at balanced positions and cannot detect the rightmost small cluster; (2) SC cannot recognize the long winding low-density regions between 2 moons because there are too many spurious edges and the Cut value along the curve is big. SC fails on $\epsilon$-graph(similar on full-RBF) because the outlier point forms a singleton cluster, and also cannot recognize the low-density curve. RMD graph significantly sparsifies the graph at low-density regions, enabling SC to cut along the winding valley, detect the small cluster and is robust to outliers. Naturally, these results depend on choices of $k$, $\sigma$, and $\epsilon$. However, our choices represent the best case scenarios for these methods and we did not see any significant improvements by varying these parameters.

\subsection{Real DataSets}\label{subsec:real}
%%%%%%%%%%%%%%%%%%%%%%%%%%%%%%%%%%%%%%%%%%
We focus on unbalanced settings and consider several real data sets. We construct $k$-NN, $b$-match, full-RBF and RMD graphs all combined with RBF weights, but do not include the $\epsilon$-graph because of its overall poor performance.
For fairness of comparison, we vary not only $\lambda$ of RMD but also $k$, $\sigma$ under the optimization step in Sec.\ref{sec:RMD_idea}. For example, the result of RBF $k$-NN graph is chosen based on optimizing the following expression:
\begin{eqnarray}
  & J(\delta)=\min_{k,\sigma}\{Cut\left(C(k,\sigma),\bar{C}(k,\sigma)\right)\} \\
\nonumber
  & s.t. ~~\min\{|C(k,\sigma)|,|\bar{C}(k,\sigma)|\}\geq \delta n
\end{eqnarray}
where, $C(k,\sigma),\,\bar C(k,\sigma)$ denotes the RatioCut partition obtained on the RBF $k$-NN graph with nearest neighbor parameter $k$ and RBF parameter $\sigma$.
The optimization problem is non-convex but involves search over a small number of parameters. We discretized the parameters in our experiments. We varied $k$ in $\{20,30,...,100\}$.
For the RBF parameter $\sigma$ it has been suggested that it should be of the same scale as the average $k$-NN distance $\tilde{d}_k$\cite{WanJebCha08}. This suggested a discretization of $\sigma$ as $2^j \tilde{d}_k$ with $j=-4,\,-3,\ldots,\,4$.
We discretized $\lambda \in [0,1]$ in steps of $0.2$. Notice that for $\lambda=1$, RMD graph is identical to $k$-NN graph. $l$ is set identical to $k$. We assume meaningful clusters are at least $5\%$ of the total number of points $\delta=0.05$. We set the GTAM parameter $\mu=0.05$\cite{JebWanCha09} for the SSL applications. For each SSL run 20 randomly labeled samples are chosen with at least one sample from each class.

\begin{figure*}[tb]
\begin{centering}
\begin{minipage}[t]{.45\textwidth}
\includegraphics[width = 1\textwidth]{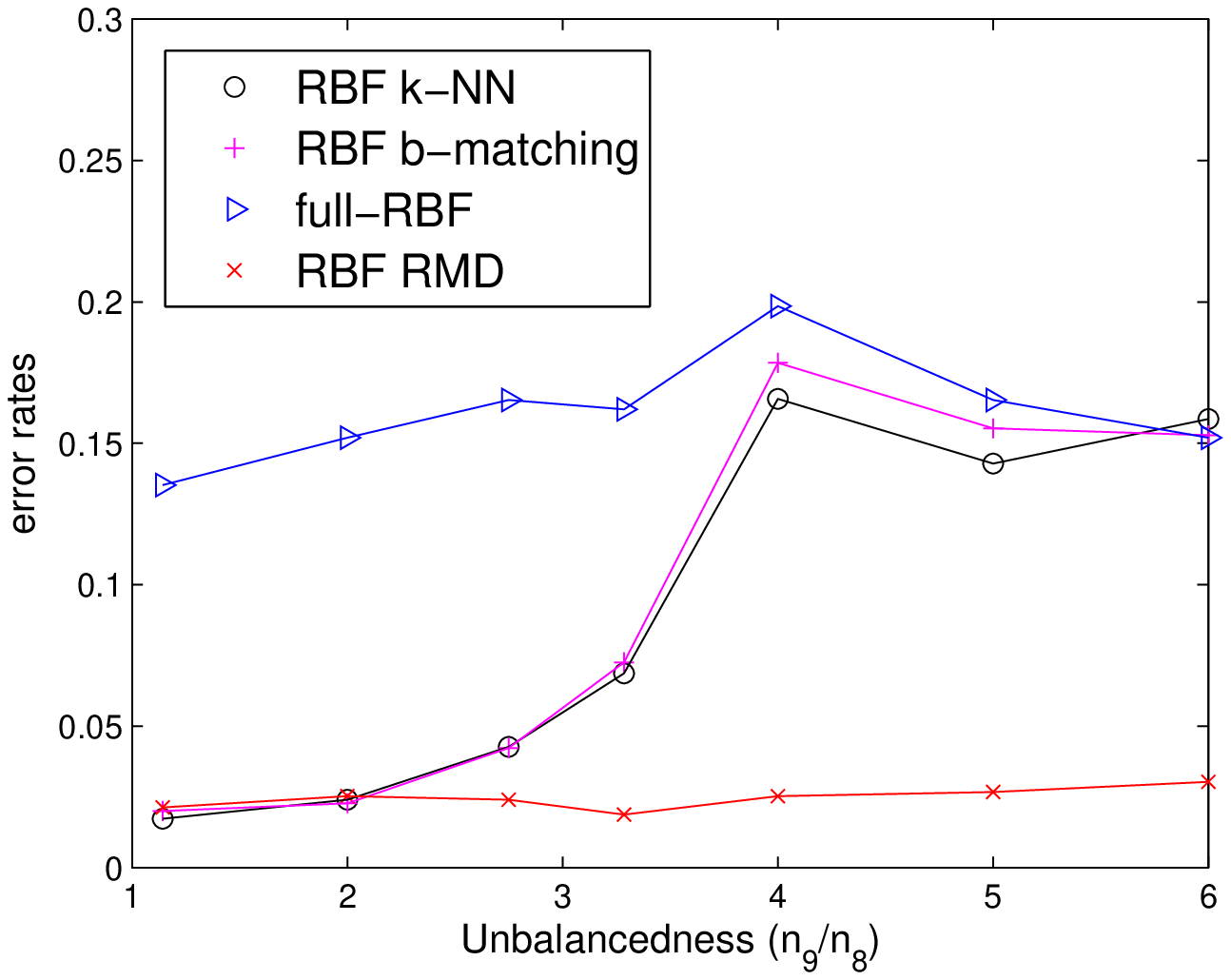}
\makebox[3.7cm]{\small (a) SC on USPS 8vs9}
\end{minipage}
\begin{minipage}[t]{.45\textwidth}
\includegraphics[width = 1\textwidth]{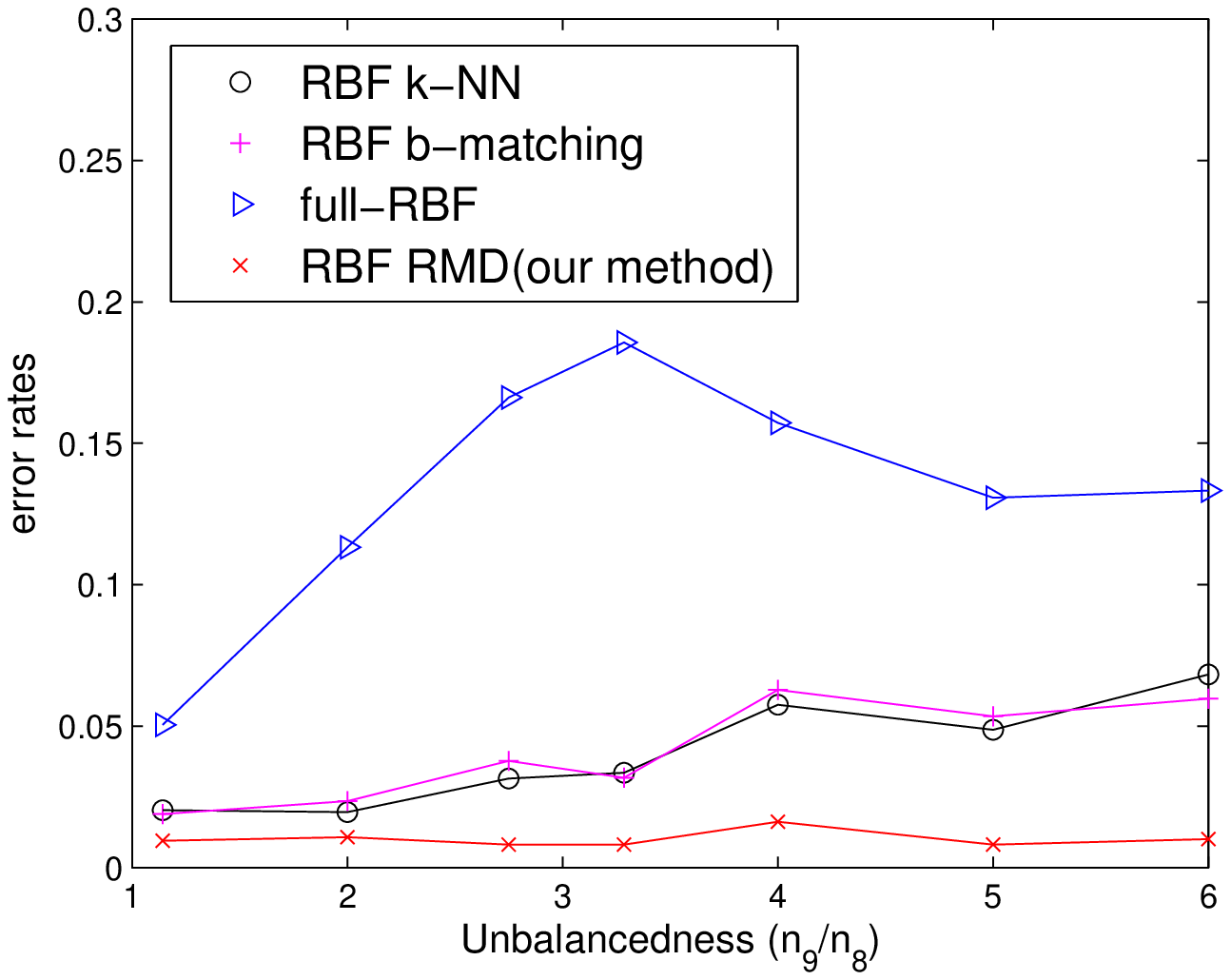}
\makebox[3.7cm]{\small (b) GTAM on USPS 8vs9}
\end{minipage}
%\begin{minipage}[t]{.34\textwidth}
%\includegraphics[width = 1\textwidth]{USPS8v9_GTAM.eps}
%\makebox[3.7cm]{\small (c) GTAM on USPS 8vs9}
%\end{minipage}
\caption{\small Error rate performance of SC and GTAM on 8vs9 of USPS digit dataset with varying levels of unbalancedness. We omitted GRF since the results are qualitatively similar. Notice that not only $\lambda$ but also $k$, $\sigma$ have been optimized. Our method adapts to different levels of unbalancedness much better than traditional graphs.}
\label{fig:USPS8v9}
\end{centering}
\end{figure*}
\textbf{Varying Unbalancedness:}
We start with a comparison for 8vs9 of the 256-dim USPS digit data set. We keep the total sample size as 750, and vary the unbalancedness, i.e. the proportion of numbers of points from two clusters, denoted by $n_8,n_9$. Fig.\ref{fig:USPS8v9} shows that as the unbalancedness increases, the performance severely degrades on traditional graphs, while our method can adapt the graph-based learning algorithms to different levels of unbalancedness very well.

\begin{table*}[tb]
\begin{center}
\begin{tabular}{|c||c|c|c|c|c|c|c|c|c|c|}
  \hline
  \multirow{2}{*}{Error Rates(\%)}   &   \multicolumn{2}{c|}{USPS}  &   \multicolumn{3}{c|}{SatImg}  &   \multicolumn{3}{c|}{OptDigit}   & \multicolumn{2}{c|}{LetterRec} \\
  \cline{2-11}
  & 8vs9 & 1,8,3,9 & 4vs3 & 3,4,5 & 1,4,7 & 9vs8 & 6vs8 & 1,4,8,9 & 6vs7 & 6,7,8 \\
  \hline\hline
  RBF $k$-NN        & 16.67 & 13.21 & 12.80 & 18.94 & 25.33 & 9.67  & 10.76  & 26.76 & 4.89 & 37.72 \\
  RBF $b$-matching  & 17.33 & 12.75 & 12.73 & 18.86 & 25.67 & 10.11  & 11.44  & 28.53 & 5.13 & 38.33 \\
  full-RBF          & 19.87 & 16.56 & 18.59 & 21.33 & 34.69 & 11.61 & 15.47 & 36.22 & 7.45 & 35.98 \\
  RBF RMD           & 4.80  & 9.18 & 7.87 & 15.26 & 19.72 & 5.43  & 6.67  & 21.35 & 2.92 & 28.68 \\
  \hline
\end{tabular}
\end{center}
\caption{\small Error rate performance of Spectral Clustering on various graphs for unbalanced real data sets. Notice that not only $\lambda$ but also $k$, $\sigma$ are optimized. Our method performs significantly better than other methods.}
\label{tab:real_SC}
\end{table*}
\begin{table*}[tb]
\begin{center}
\begin{tabular}{|c|c||c|c|c|c|c|c|c|c|c|}
  \hline
%  \multirow{2}{*}{\multicolumn{2}{c}{Error Rates(\%)}}
  \multicolumn{2}{|c||}{\multirow{2}{*}{Error Rates(\%)}}  &   \multicolumn{2}{c|}{USPS}  &   \multicolumn{2}{c|}{SatImg}  &   \multicolumn{3}{c|}{OptDigit}   & \multicolumn{2}{c|}{LetterRec} \\
  \cline{3-11}
  \multicolumn{2}{|c||}{}  & 8vs6 & 1,8,3,9 & 4vs3 & 1,4,7 & 6vs8 & 8vs9 & 6,1,8 & 6vs7 & 6,7,8 \\
  \hline\hline
  \multirow{4}{*}{GRF}
    & RBF $k$-NN            & 5.70 & 13.29 & 14.64 & 16.68 & 5.68  & 7.57  & 7.53 & 7.67 & 28.33 \\
    & RBF $b$-matching      & 6.02 & 13.06 & 13.89 & 16.22 & 5.95  & 7.85  & 7.92 & 7.82 & 29.21 \\
    & full-RBF              & 15.41 & 12.37 & 14.22 & 17.58 & 5.62 & 9.28 & 7.74 & 11.52 & 28.91 \\
    & RBF RMD               & 1.08  & 10.24 & 9.74 & 15.04 & 2.07  & 2.30  & 5.82 & 5.23 & 27.24 \\
  \hline
  \multirow{4}{*}{GTAM}
    & RBF $k$-NN            & 4.11  & 10.88 & 26.63 & 20.68 & 11.76 & 5.74  & 12.68 & 19.45 & 27.66 \\
    & RBF $b$-matching      & 3.96  & 10.83 & 27.03 & 20.83 & 12.48 & 5.65  & 12.28 & 18.85 & 28.01 \\
    & full-RBF              & 16.98  & 11.28 & 18.82 & 21.16 & 13.59 & 7.73 & 13.09 & 18.66 & 30.28 \\
    & RBF RMD               & 1.22  & 9.13 & 18.68 & 19.24 & 5.81  & 3.12  & 10.73 & 15.67 & 25.19 \\
  \hline
\end{tabular}
\end{center}
\caption{\small Error rate performance of GRF and GTAM on various graphs for unbalanced real data sets. Notice that not only $\lambda$ but also $k$, $\sigma$ are optimized to achieve best performance. Our method performs significantly better than other methods.}
\label{tab:real_SSL}
\end{table*}
\textbf{Other Real Data Sets:}
We apply SC and SSL algorithms on several other real data sets including USPS, waveform database generator(21-dim), Statlog landsat satellite images(36-dim), letter recognition images(16-dim) and optical recognition of handwritten digits(64-dim) \cite{uci10}.
We fix 150/600, 200/400/600, 200/300/400/500 samples for 2,3,4-class cases, with corresponding orders of class indices listed in Tab.\ref{tab:real_SC},\ref{tab:real_SSL}.
Tab.\ref{tab:real_SC},\ref{tab:real_SSL} shows that even when $k$ and $\sigma$ for RBF $k$-NN($b$-matching) and full-RBF graphs are optimized to achieve optimal performance, RMD graph still consistently outperforms other methods.

\subsection{Applications to Small Cluster Detection}
We illustrate how our method can be used to find small-size clusters. This type of problem arises in community detection in large real networks, where graph-based approaches are popular but small-size community detection is difficult \cite{Shah10}.

Our synthetic dataset depicted in Fig.~\ref{fig:multcluster} has 1 large and 2 small proximal Gaussian components along $x_1$ axis: $\sum^{3}_{i=1}\alpha_iN(\mu_i,\Sigma_i)$, where $\alpha_1:\alpha_2:\alpha_3=2:8:1$, $\mu_1$=[-0.7;0], $\mu_2$=[4.5;0], $\mu_3$=[9.7;0], $\Sigma_1=I, \Sigma_2=diag(2,1), \Sigma_3=0.7I$.

Fig.\ref{fig:multiple_cuts}(a) shows a plot of cut values for different cut positions averaged over 20 Monte Carlo runs. We note that the cut-value plot resembles the underlying density. Two density valleys are both at the unbalanced positions. The rightmost cluster is smaller than the left cluster, but has a deeper valley.

\begin{wrapfigure}{r}{.5\textwidth}
\vspace{-20pt}
\centering
\includegraphics[width=.47\textwidth]{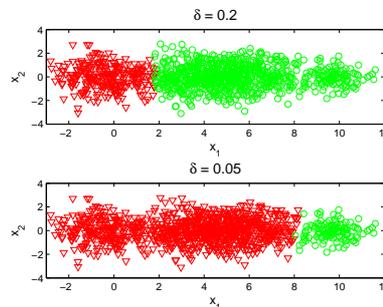}
\vspace{-10pt}
\caption{\small Gaussian mixture with three unbalanced Gaussian components. Results of our method is depicted for a single realization. Our method is able to discover two small clusters. The larger cluster is detected for a larger value of $\delta$ and the smaller cluster is detected for a smaller $\delta$ value(see Eq.~\ref{eq:selection}).}
\label{fig:multcluster}
\end{wrapfigure}
To apply our method we vary the cluster-size threshold $\delta$ in Eq.(\ref{eq:selection}). We can now plot the Cut-value against $\delta$ as shown in Fig.\ref{fig:multiple_cuts}(b). As seen in Fig.\ref{fig:multiple_cuts}(b), when $\delta\geq 0.3$, the optimal cut is close to the valley. However, since the proportion of data samples in the smaller clusters is less than 30\% we see that the optimal cut is bounded away from both valleys. As $\delta$ is further decreased, namely, in the range $0.25\geq\delta\geq0.15$, the optimal cut is now attained at the left valley($x_1\approx 1.8$). An interesting phenomena is that the curve flattens out in this range.
This corresponds to the fact that the cut value is minimized at this position ($x_1 = 1.8$) for any value of $\delta \in [.15,\,.25]$. This flattening out can happen only at valleys since valleys represent a ``local'' minima for the optimization step of Eq.~\ref{eq:selection} under the constraint imposed by $\delta$. Consequently, small clusters can be detected based on the flat spots. Next when we further vary $\delta$ in the region $0.1\geq\delta\geq0.05$, the best cut is attained near the right and deeper valley($x_1\approx 8.2$). Again the curve flattens out revealing another small cluster.
\begin{figure}[tb]
%\begin{centering}
%\begin{minipage}[t]{.5\textwidth}
\centering
\subfigure[\small Cut value vs. cut position]{
\includegraphics[width = 5cm]{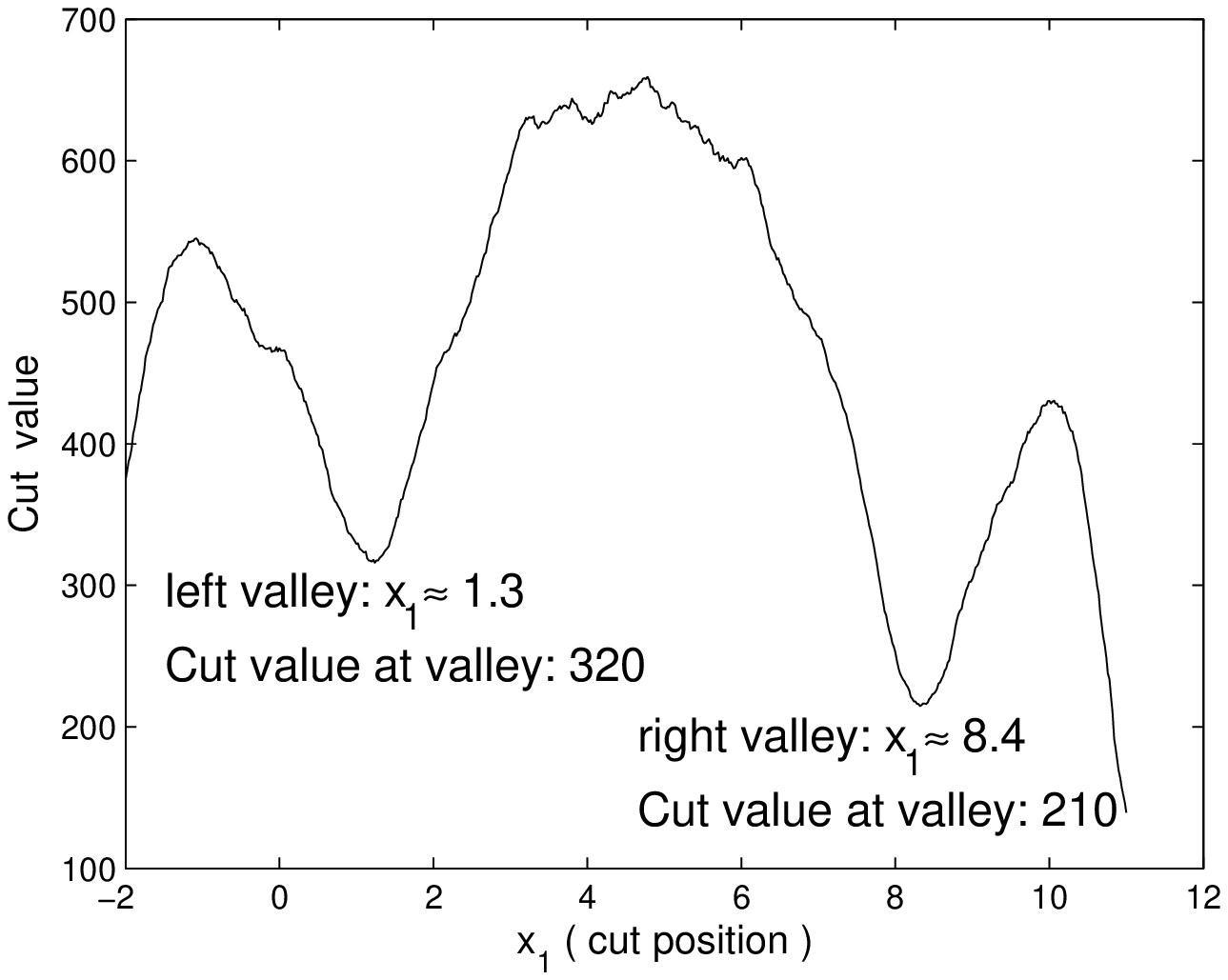}}
%\makebox[4cm]{\small (a) Cut value vs. cut position}
%\end{minipage}
%\begin{minipage}[t]{.5\textwidth}
\subfigure[\small \small Cut value vs. Cluster size($\delta$)]{
\includegraphics[width = 5cm]{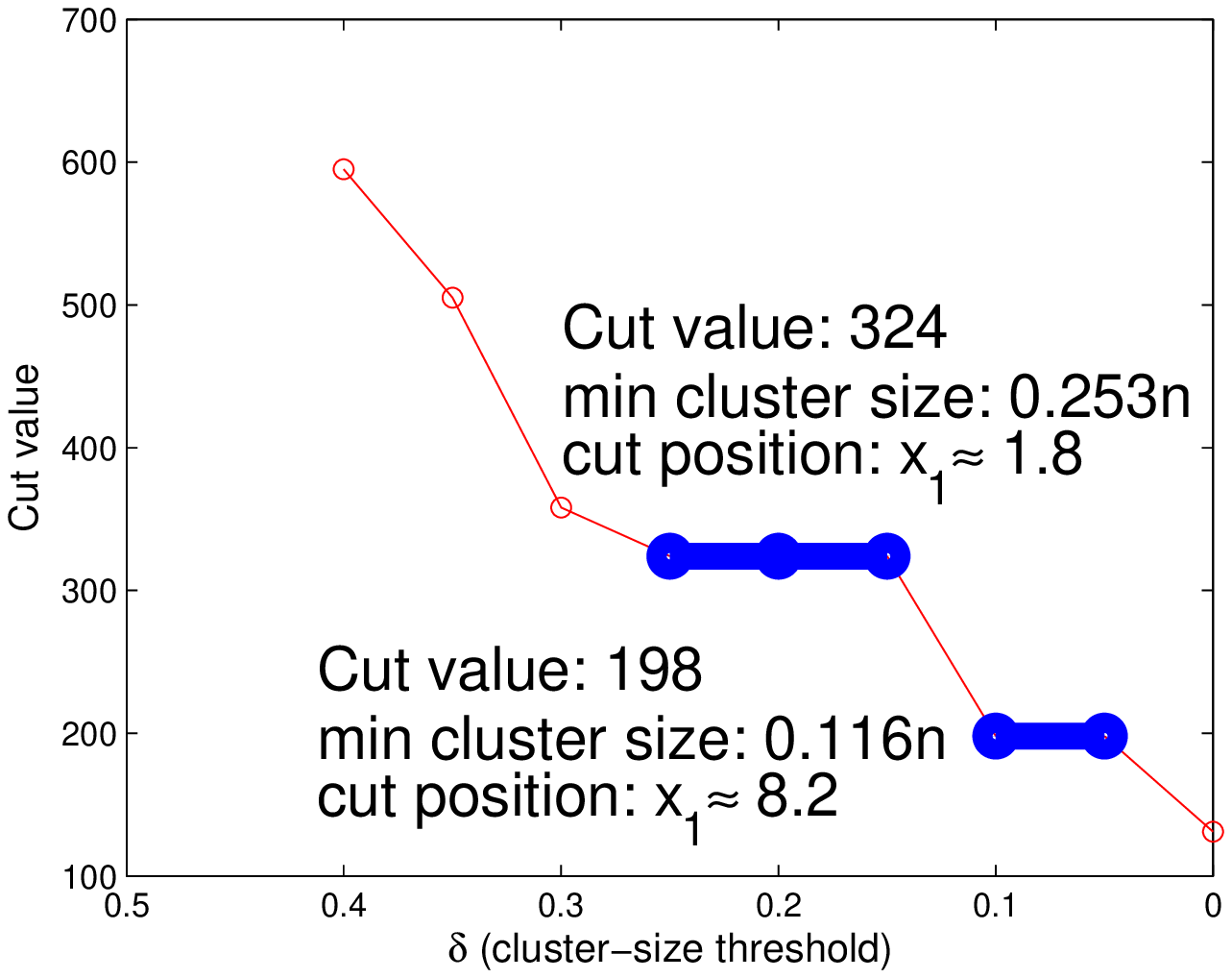}}
%\makebox[4cm]{\small (b) Cut value vs. $\delta$}
%\end{minipage}
\caption{\small 2-clustering results of 1 large and 2 small proximal gaussian mixture components. Both valleys are at unbalanced positions. The rightmost cluster is smaller than the left cluster with a deeper valley. $n=1100$, binary weight is adopted. Cut values in (a) are averaged over 20 Monte Carlo runs. Results in (b) are from one run. By varying cluster-size allowance threshold $\delta$, our method is able to detect different small clusters and generate meaningful cuts.}
\label{fig:multiple_cuts}
%\end{centering}
\vspace*{-0.2in}
\end{figure}

%%%%%%%%%%%%%%%%%%%%%%%%%%%%%%%%%%%%%%%%%%
\subsection{Comments on RMD Method}\label{subsec:discRMD}
%%%%%%%%%%%%%%%%%%%%%%%%%%%%%%%%%%%%%%%%%%
{\bf Tuning Parameters:}
We first describe parameters involved in our RMD method. We have already pointed out that $\lambda$ is a parameter that is optimized and so does not count as a tuning parameter. So we are left with parameters $l$ and $\delta$. As we pointed out in Sec.~3 the choice of $\delta$ is based on our prior or desire to find clusters that are sizable, say 5\% to 10\% of the data. This leaves the choice to a single tuning parameter, namely, $l$. Our method appears to be relatively insensitive to choice of $l$. Note that unlike $k$ and $\sigma$, which are used for graph construction, the parameter $l$ here is primarily used to relatively order data points based on whether they belong to high-density or low-density regions. In most situations we have encountered this ranking does not substantially change, namely, it is rarely the case where an empirically low ranked data point should have a high-rank (i.e. high-density region). Similar results have also been observed in the context of high-dimensional anomaly detection~\cite{zhaonips,zhaoaistats12}.

%This is reasonable because different from $k$ or $\sigma$ which is directly used in constructing the graph, $l$ is used in computing the ranks of data points, representing the neighborhood size. The ranks only concern relative orderings among points, whatever the neighborhood size we consider for every point.
%$\lambda$ is automatically picked for a given $\delta$, which represents the smallest cluster size allowance. Smaller $\lambda$ can find deeper valleys, or even the boundaries(the rightmost point in Fig.\ref{fig:multiple_cuts}(b)); then the optimization step can remove those outlier results and pick the deepest valley under a certain specified $\delta$. All the parameters are simple and explicit, and do not need careful tuning.

{\bf Time Complexity:}
The time complexity of U-statistic rank computation is $O(Bdn^2logn)$, and RMD graph construction is $O(dn^2logn)$, which leads to an aggregate complexity of $O\left((B+1)dn^2logn\right)$. In experiments we set $B=5$, so the complexity is on the same order of constructing a $k$-NN graph($O(dn^2logn)$).

%In experiments we also find setting the GTAM parameter $\mu$ small leads to better SSL performance. For GTAM $\mu$ controls the tradeoff between the global geometric structure of unlabeled data (minimizing RatioCut) and the confidence of labeled points (penalty of label inconsistency). Small $\mu$ means we would rather trust the global structure rather than the labeled points. This is probably because with few labeled points, it may be better to trust the global geometric structure of data(unsupervised graph partitioning) rather than being guided by the few possibly misleading labeled points.

%%%%%%%%%%%%%%%%%%%%%%%%%%%%%%%%%%%%%%%%%%
\section{Conclusions}\label{sec:conclusion}
%%%%%%%%%%%%%%%%%%%%%%%%%%%%%%%%%%%%%%%%%%
We have demonstrated that spectral clustering and graph based semi-supervised learning algorithms can fail on conventional graph methods for unbalanced and proximal data clusters. We propose a systematic procedure for graph construction (RMD graph), based on adaptive sparsification and densification of neighborhoods of $k$-NN graphs. Our method effectively incorporates density, maintains robustness to outliers, and adapts to different degrees of unbalancedness. We present a optimization framework for graph-based approaches, which allows for best sizable clusters separated by the smallest cut value. By constraining the smallest cluster sizes we can detect multiple small clusters and generate different meaningful cuts. Our simulations demonstrate significant performance improvements over existing methods for synthetic and real datasets. The ability to detect small-size clusters (Fig.\ref{fig:multiple_cuts}) indicates that our idea may be utilized in other applications such as community detection in large real networks, where graph-based approaches are popular but small-size community detection is difficult \cite{Shah10}. %Another method is also applicable setting is that for a densely-connected network with many attributes associated with each node, our scheme can sparsify the graph differently based on different subsets of attributes of interest, allowing multiple interpretations of the network.

\section*{Appendix: Proofs of Theorems}
For ease of development, let $n=m_1(m_2+1)$, and divide $n$ data points into: $D=D_0 \bigcup  D_1 \bigcup ... \bigcup D_{m_1}$, where $D_0=\{x_1,...,x_{m_1}\}$, and each $D_j, j=1,...,m_1$ involves $m_2$ points. $D_j$ is used to generate the statistic $G$ for $u$ and $x_j\in D_0$, for $j=1,...,m_1$. $D_0$ is used to compute the rank of $u$:
\begin{equation}
    R(u) = \frac{1}{m_1}\sum_{j=1}^{m_1} \mathbb{I}_{\{ G(x_j;D_j)>G(u;D_j) \}}
\end{equation}
We provide the proof for the statistic $G(u)$ of the following form:
\begin{eqnarray}
  G(u;D_j) &=& \frac{1}{l}\sum^{l+\lfloor \frac{l}{2} \rfloor}_{i=l-\lfloor \frac{l-1}{2} \rfloor}\left( \frac{l}{i} \right)^{\frac{1}{d}}D_{(i)}(u).
\end{eqnarray}
where $D_{(i)}(u)$ denotes the distance from $u$ to its $i$-th nearest neighbor among $m_2$ points in $D_j$. Practically we can omit the weight as Eq.(\ref{eq:grank}) in the paper. The proof for the first and second statistics can be found in \cite{zhaonips}.

\textbf{Proof of Theorem 1:}
\begin{proof}
The proof involves two steps:
\begin{itemize}
  \item[1.] The expectation of the empirical rank $\mathbb{E}\left[R(u)\right]$ is shown to converge to $p(u)$ as $n\rightarrow\infty$.
  \item[2.] The empirical rank $R(u)$ is shown to concentrate at its expectation as $n\rightarrow\infty$.
\end{itemize}
The first step is shown through Lemma \ref{lem:expectation}. For the second step, notice that the rank $R(u) = \frac{1}{m_1}\sum_{j=1}^{m_1} Y_j$, where $Y_j = \mathbb{I}_{\{ G(x_j;D_j)>G(u;D_j) \}}$ is independent across different $j$'s, and $Y_j \in [0,1]$. By Hoeffding's inequality, we have:
\begin{equation}
    \mathbb{P}\left( | R(u) - \mathbb{E}\left[R(u)\right] | > \epsilon \right) < 2\exp\left( -2m_1\epsilon^2 \right)
\end{equation}
Combining these two steps finishes the proof.
\end{proof}

\textbf{Proof of Theorem 2:}
\begin{proof}
We only present a brief outline of the proof. We want to establish the convergence result of the cut term and the balancing terms respectively, that is:
\begin{eqnarray}
    &\frac{1}{nk_n}\sqrt[d]{\frac{n}{k_n}}cut_n(S)
    \rightarrow C_d\int_S{f^{1-\frac{1}{d}}(s)\rho(s)^{1+\frac{1}{d}}ds}. \label{eq:term1}\\
    &nk_n\frac{1}{vol(V^\pm)}\rightarrow
    \frac{1}{\mu(C^\pm)}. \label{eq:term2}
    %n\frac{1}{|V^\pm|} &\rightarrow \frac{1}{2\mu(C^{\pm})} \label{eq:term3}
\end{eqnarray}
where $V^+(V^-)=\{x\in{V}: x\in{C^+}(C^-)\}$ are the discrete version of $C^+(C^-)$.

Eq.(\ref{eq:term1}) is established in two steps. First we can show that the LHS cut term converges to its expectation $\mathbb{E}\left(\frac{1}{nk_n}\sqrt[d]{\frac{n}{k_n}}cut_n(S)\right)$ by making use of the McDiarmid's inequality. Second we show that this expectation term actually converges to the RHS of
Eq.(\ref{eq:term1}). This is the most intricate part and we state it as a separate result in Lemma~\ref{expectation}.

%To establish \ref{eq:term3}, the idea is that  the number of points in $V^+$ is binomially distributed $\text{Binom}(n,\mu(C^+))$. Using the Chernoff bound of binomial sum we can show that almost surely Equation \ref{eq:term3} holds true.

For Eq.(\ref{eq:term2}), recall that the volume term of $V^+$ is $vol(V^+)=\sum_{u\in{V^+},v\in{V}}1$. It can be shown that as $n\rightarrow\infty$, the distance between any connected pair $(u,v)$ goes to zero. Next we note that  the number of points in $V^+$ is binomially distributed $Binom(n,\mu(C^+))$. Using the Chernoff bound of binomial sum we can show that almost surely Equation \ref{eq:term2} holds true.
\end{proof}

\begin{lemma}\label{expectation}
Given the assumptions of Theorem 2,
\begin{equation}
    \mathbb{E}\left(\frac{1}{nk_n}\sqrt[d]{\frac{n}{k_n}}cut_n(S)\right)\longrightarrow C_d\int_S{f^{1-\frac{1}{d}}(s)\rho(s)^{1+\frac{1}{d}}ds}.
\end{equation}
where $C_d=\frac{2\eta_{d-1}}{(d+1)\eta_d^{1+1/d}}$.
\end{lemma}

\begin{proof}
The proof is similar to \cite{Maier2} and we provide an outline here. The first trick is to define a cut function for a fixed point $x_i\in V^+$, whose expectation is easier to compute:
\begin{eqnarray}
cut_{x_i} = \sum_{v\in V^{-},(x_i,v)\in E}w(x_i,v).
\end{eqnarray}
Similarly, we can define $cut_{x_i}$ for $x_i\in V^-$. The expectation of $cut_{x_i}$ and  $cut_n(S)$ can be related:
\begin{eqnarray}\label{eq:expect}
\mathbb{E}(cut_n(S))=n\mathbb{E}_x(\mathbb{E}(cut_{x}))
\end{eqnarray}
Then the value of $\mathbb{E}(cut_{x_i})$ can be computed as,
\begin{equation}
    (n-1)\int_0^{\infty}{\left[\int_{B(x_i,r)\cap{C^-}}f(y)dy\right]dF_{R_{x_i}^k}(r)}.
\end{equation}
where $r$ is the distance of $x_i$ to its $k_n\rho(x_i)$-th nearest neighbor. The value of $r$ is a random variable and can be characterized by the CDF $F_{R_{x_i}^k}(r)$.
Combining equation \ref{eq:expect} we can write down the whole expected cut value
\begin{eqnarray}
  \mathbb{E}(cut_n(S)) =n\mathbb{E}_x(\mathbb{E}(cut_{x}))= n\int_{\mathbb{R}^d}f(x)\mathbb{E}(cut_{x})dx \\
   = n(n-1)\int_{\mathbb{R}^d}f(x)\left[\int_0^{\infty}{g(x,r)dF_{R_x^k}(r)}\right]dx.
\end{eqnarray}

To simplify the expression, we use $g(x,r)$ to denote
\begin{equation}
    g(x,r)=\begin{cases}
               \int_{B(x,r)\cap{C^-}}f(y)dy,  x\in{C^+} \\
               \int_{B(x,r)\cap{C^+}}f(y)dy,  x\in{C^-}.
             \end{cases}
\end{equation}

Under general assumptions, when $n$ tends to infinity, the random variable $r$ will highly concentrate around its mean $\mathbb{E}(r_x^k)$.
Furthermore, as $k_n/n\rightarrow{0}$, $\mathbb{E}(r_x^k)$ tends to zero and the speed of convergence
\begin{eqnarray}\label{eq:EkNN}
\mathbb{E}(r_x^k)\approx(k\rho(x)/((n-1)f(x)\eta_d))^{1/d}
\end{eqnarray}
So the inner integral in the cut value can be approximated by $g(x,\mathbb{E}(r_x^k))$, which implies,
\begin{equation}
    \mathbb{E}(cut_n(S))\approx{n}(n-1)\int_{\mathbb{R}^d}f(x)g(x,\mathbb{E}(r_x^k))dx.
\end{equation}

The next trick is to decompose the integral over $\mathbb{R}^d$ into two orthogonal directions, i.e., the direction along the hyperplane $S$ and its normal direction (We use $\overrightarrow{n}$ to denote the unit normal vector):
\begin{equation}
    \int_{\mathbb{R}^d}f(x)g(x,\mathbb{E}(r_x^k))dx= \\
    \int_{S}\int_{-\infty}^{+\infty}f(s+t\overrightarrow{n})g(s+t\overrightarrow{n},\mathbb{E}(r_{s+t\overrightarrow{n}}^k))dtds.
\end{equation}
When $t>\mathbb{E}(r_{s+t\overrightarrow{n}}^k)$, the integral region of $g$ will be empty: $B(x,\mathbb{E}(r_x^k))\cap{C^-}=\emptyset$. On the other hand, when $x=s+t\overrightarrow{n}$ is close to $s\in{S}$, we have the approximation $f(x)\approx{f(s)}$:
\begin{eqnarray}
  &\int_{-\infty}^{+\infty}f(s+t\overrightarrow{n})g(s+t\overrightarrow{n},\mathbb{E}(r_{s+t\overrightarrow{n}}^k))dt \\
  &\approx 2\int_{0}^{\mathbb{E}(r_{s}^k)}f(s)\left[f(s)vol\left(B(s+t\overrightarrow{n},\mathbb{E}{r_s^k})\cap{C^-}\right)\right]dt  \\
  &= 2f^2(s)\int_{0}^{\mathbb{E}(r_{s}^k)}vol\left(B(s+t\overrightarrow{n},\mathbb{E}(r_s^k))\cap{C^-}\right)dt.
\end{eqnarray}

The term $vol\left(B(s+t\overrightarrow{n},\mathbb{E}(r_s^k))\cap{C^-}\right)$ is the volume of $d$-dim spherical cap of radius $\mathbb{E}(r_s^k))$, which is at distance $t$ to the center. Through direct computation we obtain:
\begin{equation}
    \int_{0}^{\mathbb{E}(r_{s}^k)}vol\left(B(s+t\overrightarrow{n},\mathbb{E}(r_s^k))\cap{C^-}\right)dt=\mathbb{E}(r_s^k)^{d+1}\frac{\eta_{d-1}}{d+1}.
\end{equation}
Combining the above step and plugging in the approximation of $\mathbb{E}(r_s^k)$ in Eq.(\ref{eq:EkNN}), we finish the proof.
%The $k$-NN radius tends to 0, so for point $x$ and its linked neighbor $y$, $p(x)+p(y)\approx2p(x)$. Decompose the integration over $\mathbb{R}^d$ into two steps, first at point $s$ over $S$ and then along the orthogonal direction at $s$, and insert the approximation of $k$-NN radius at $s$, we can obtain the result.
\end{proof}

\begin{lemma}\label{lem:expectation}
By choosing $l$ properly, as $m_2\rightarrow\infty$, it follows that,
$$ | \mathbb{E}\left[R(u)\right] - p(u)| \longrightarrow 0$$
\end{lemma}
\begin{proof}
Take expectation with respect to $D$:
\begin{eqnarray}
\mathbb{E}_D\left[R(u)\right]
&=&\mathbb{E}_{D\backslash D_0}\left[\mathbb{E}_{D_0}\left[\frac{1}{m_1}\sum_{j=1}^{m_1}
 \mathbb{I}_{\{G(u;D_j)<G(x_j;D_j)\}}\right]\right]\\
&=&\frac{1}{m_1}\sum_{j=1}^{m_1}\mathbb{E}_{x_j}\left[
\mathbb{E}_{D_j}\left[
\mathbb{I}_{\{G(u;D_j)<G(x_j;D_j)\}}\right]\right]\\
&=&\mathbb{E}_x\left[\mathcal{P}_{D_1}\left(G(u;D_1)<G(x;D_1)\right)\right]
\end{eqnarray}
The last equality holds due to the i.i.d symmetry of $\{x_1,...,x_{m_1}\}$ and $D_1,...,D_{m_1}$. We fix both $u$ and $x$ and temporarily discarding $\mathbb{E}_{D_1}$. Let $F_x(y_1,...,y_{m_2})=G(x)-G(u)$, where $y_1,...,y_{m_2}$ are the $m_2$ points in $D_1$. It follows:
\begin{equation}
    \mathcal{P}_{D_1}\left(G(u)<G(x)\right)
    =\mathcal{P}_{D_1}\left(F_x(y_1,...,y_{m_2})>0\right)
    =\mathcal{P}_{D_1}\left(F_x-\mathbb{E}F_x>-\mathbb{E}F_x\right).
\end{equation}

To check McDiarmid's requirements, we replace $y_j$ with $y_j'$. It is easily verified that $\forall j=1,...,m_2$,
\begin{equation}\label{equ:mcdiarmid_condition}
    |F_x(y_1,...,y_{m_2})-F_x(y_1,...,y_j',...,y_{m_2})| \leq 2^{\frac{1}{d}}\frac{2C}{l} \leq \frac{4C}{l}
\end{equation}
where $C$ is the diameter of support. Notice despite the fact that $y_1,...,y_{m_2}$ are random vectors we can still apply MeDiarmid's inequality, because according to the form of $G$, $F_x(y_1,...,y_{m_2})$ is a function of $m_2$ i.i.d random variables $r_1,...,r_{m_2}$ where $r_i$ is the distance from $x$ to $y_i$.
Therefore if $\mathbb{E}F_x<0$, or $\mathbb{E}G(x)<\mathbb{E}G(u)$, we have by McDiarmid's inequality,
\begin{equation}
    \mathcal{P}_{D_1}\left(G(u)<G(x)\right)
    = \mathcal{P}_{D_1}\left( F_x > 0 \right)
    = \mathcal{P}_{D_1}\left( F_x-\mathbb{E}F_x>-\mathbb{E}F_x \right)
    \leq \exp\left(-\frac{(\mathbb{E}F_x)^2 l^2}{8C^2m_2}\right)
\end{equation}
Rewrite the above inequality as:
\begin{equation}\label{equ:bound_no_expectation}
    \mathbb{I}_{\{\mathbb{E}F_x>0\}}-e^{-\frac{(\mathbb{E}F_x)^2 l^2}{8C^2m_2}}
    \leq \mathcal{P}_{D_1}\left( F_x > 0 \right)
    \leq \mathbb{I}_{\{\mathbb{E}F_x>0\}}+e^{-\frac{(\mathbb{E}F_x)^2 l^2}{8C^2m_2}}
\end{equation}
It can be shown that the same inequality holds for $\mathbb{E}F_x>0$, or $\mathbb{E}G(x)>\mathbb{E}G(u)$. Now we take expectation with respect to $x$:
\begin{equation}\label{equ:bound_with_expectation}
    \mathcal{P}_x\left(\mathbb{E}F_x>0\right)-\mathbb{E}_x\left[e^{-\frac{(\mathbb{E}F_x)^2 l^2}{8C^2m_2}}\right] \leq
    \mathbb{E}\left[\mathcal{P}_{D_1}\left( F_x > 0 \right)\right] \leq \mathcal{P}_x\left(\mathbb{E}F_x>0\right)+\mathbb{E}_x\left[e^{-\frac{(\mathbb{E}F_x)^2 l^2}{8C^2m_2}}\right]
\end{equation}
Divide the support of $x$ into two parts, $\mathbb{X}_1$ and $\mathbb{X}_2$, where $\mathbb{X}_1$ contains those $x$ whose density $f(x)$ is relatively far away from $f(u)$, and $\mathbb{X}_2$ contains those $x$ whose density is close to $f(u)$. We show for $x \in \mathbb{X}_1$, the above exponential term converges to 0 and $\mathcal{P}\left(\mathbb{E}F_x>0\right) = \mathcal{P}_x\left( f(u)>f(x) \right)$, while the rest $x\in\mathbb{X}_2$ has very small measure. Let $A(x)=\left(\frac{k}{f(x) c_d m_2}\right)^{1/d}$. By Lemma \ref{lem:bound_expectation} we have:
\begin{equation}
    | \mathbb{E}G(x) - A(x) | \leq \gamma \left(\frac{l}{m_2}\right)^{\frac{1}{d}} A(x)
    \leq \gamma \left(\frac{l}{m_2}\right)^{\frac{1}{d}} \left(\frac{l}{f_{min}c_d m_2}\right)^{\frac{1}{d}}
    =    \left(\frac{\gamma_1}{c_d^{1/d}}\right) \left(\frac{l}{m_2}\right)^{\frac{2}{d}}
\end{equation}
where $\gamma$ denotes the big $O(\cdot)$, and $\gamma_1 = \gamma \left(\frac{1}{f_{min}}\right)^{1/d}$. Applying uniform bound we have:
\begin{equation}
    A(x)-A(u)- 2\left(\frac{\gamma_1}{c_d^{1/d}}\right) \left(\frac{l}{m_2}\right)^{\frac{2}{d}}
    \leq \mathbb{E}\left[G(x) - G(u)\right]
    \leq A(x)-A(u)+ 2\left(\frac{\gamma_1}{c_d^{1/d}}\right) \left(\frac{l}{m_2}\right)^{\frac{2}{d}}
\end{equation}
Now let $\mathbb{X}_1=\{ x:|f(x)-f(u)|\geq 3\gamma_1 d f_{min}^{\frac{d+1}{d}} \left(\frac{l}{m_2}\right)^{\frac{1}{d}} \}$. For $x\in \mathbb{X}_1$, it can be verified that $|A(x)-A(u)|\geq 3\left(\frac{\gamma_1}{c_d^{1/d}}\right) \left(\frac{l}{m_2}\right)^{\frac{2}{d}}$, or $|\mathbb{E}\left[G(x) - G(u)\right]| > \left(\frac{\gamma_1}{c_d^{1/d}}\right) \left(\frac{l}{m_2}\right)^{\frac{2}{d}}$, and $\mathbb{I}_{\{f(u)>f(x)\}}=\mathbb{I}_{\{\mathbb{E}G(x)>\mathbb{E}G(u)\}}$. For the exponential term in Equ.(\ref{equ:bound_no_expectation}) we have:
\begin{equation}
    \exp\left(-\frac{(\mathbb{E}F_x)^2 l^2}{2C^2m_2}\right)
    \leq \exp\left(-\frac{ \gamma_1^2 l^{2+\frac{4}{d}} }{ 8C^2 c_d^{\frac{2}{d}} m_2^{1+\frac{4}{d}} } \right)
\end{equation}
For $x\in \mathbb{X}_2=\{x:|f(x)-f(u)|< 3\gamma_1 d \left(\frac{l}{m_2}\right)^{\frac{1}{d}}f_{min}^{\frac{d+1}{d}} \}$, by the regularity assumption, we have $\mathcal{P}(\mathbb{X}_2)<3M\gamma_1 d \left(\frac{l}{m_2}\right)^{\frac{1}{d}}f_{min}^{\frac{d+1}{d}}$. Combining the two cases into Equ.(\ref{equ:bound_with_expectation}) we have for upper bound:
\begin{eqnarray}
  \mathbb{E}_D\left[R(u)\right]
  &=& \mathbb{E}_x\left[\mathcal{P}_{D_1}\left(G(u)<G(x)\right)\right] \\
  &=& \int_{\mathbb{X}_1}\mathcal{P}_{D_1}\left(G(u)<G(x)\right)f(x)dx +  \int_{\mathbb{X}_2}\mathcal{P}_{D_1}\left(G(u)<G(x)\right)f(x)dx \\
  &\leq& \left( \mathcal{P}_x\left(f(u)>f(x)\right) + \exp\left(-\frac{ \gamma_1^2 l^{2+\frac{4}{d}} }{ 8C^2 c_d^{\frac{1}{d}} m_2^{1+\frac{4}{d}} } \right) \right)\mathcal{P}(x\in \mathbb{X}_1) + \mathcal{P}(x\in \mathbb{X}_2) \\
  &\leq&  \mathcal{P}_x\left(f(u)>f(x)\right) + \exp\left(-\frac{ \gamma_1^2 l^{2+\frac{4}{d}} }{ 8C^2 c_d^{\frac{1}{d}} m_2^{1+\frac{4}{d}} } \right) + 3M\gamma_1 d f_{min}^{\frac{d+1}{d}} \left(\frac{l}{m_2}\right)^{\frac{1}{d}}
\end{eqnarray}
Let $l=m_2^\alpha$ such that $\frac{d+4}{2d+4}<\alpha<1$, and the latter two terms will converge to 0 as $m_2 \rightarrow \infty$. Similar lines hold for the lower bound. The proof is finished.
\end{proof}

\begin{lemma}\label{lem:bound_expectation}
Let $A(x)=\left(\frac{l}{m c_d f(x)}\right)^{1/d}$, $\lambda_1 = \frac{\lambda}{f_{min}}\left(\frac{1.5}{c_d f_{min}}\right)^{1/d}$. By choosing $l$ appropriately, the expectation of $l$-NN distance $\mathbb{E}D_{(l)}(x)$ among $m$ points satisfies:
\begin{equation}
    | \mathbb{E}D_{(l)}(x) - A(x) | = O\left( A(x) \lambda_1 \left(\frac{l}{m}\right)^{1/d} \right)
\end{equation}
\end{lemma}

\begin{proof}
Denote $r(x,\alpha)=\min\{r:\mathcal{P}\left(B(x,r)\right)\geq \alpha\}$. Let $\delta_m \rightarrow 0$ as $m \rightarrow \infty$, and $0<\delta_{m}<1/2$.
Let $U\sim Bin(m,(1+\delta_m)\frac{l}{m})$ be a binomial random variable, with $\mathbb{E}U = (1+\delta_{m})l$. We have:
\begin{eqnarray}
  \mathcal{P}\left(D_{(l)}(x)>r(x,(1+\delta_m)\frac{l}{m})\right)
  &=& \mathcal{P}\left(U<l\right) \\
  &=& \mathcal{P}\left(U<\left(1-\frac{\delta_m}{1+\delta_m}\right)(1+\delta_m)l\right) \\
  &\leq& \exp\left(-\frac{\delta_m^2 l}{2(1+\delta_m)}\right)
\end{eqnarray}
The last inequality holds from Chernoff's bound. Abbreviate $r_1 = r(x,(1+\delta_m)\frac{l}{m})$, and $\mathbb{E}D_{(l)}(x)$ can be bounded as:
\begin{eqnarray}
  \mathbb{E}D_{(l)}(x)
  &\leq& r_1\left[1-\mathcal{P}\left(D_{(l)}(x)>r_1\right)\right] + C\mathcal{P}\left(D_{(l)}(x)>r_1\right)  \\
  &\leq& r_1 + C \exp\left(-\frac{\delta_m^2 l}{2(1+\delta_m)}\right)
\end{eqnarray}
where $C$ is the diameter of support. Similarly we can show the lower bound:
\begin{equation}
    \mathbb{E}D_{(l)}(x) \geq r(x,(1-\delta_m)\frac{l}{m}) - C \exp\left(-\frac{\delta_m^2 l}{2(1-\delta_m)}\right)
\end{equation}
Consider the upper bound. We relate $r_1$ with $A(x)$. Notice $\mathcal{P}\left(B(x,r_1)\right)=(1+\delta_m)\frac{l}{m} \geq c_d r_1^d f_{min}$, so a fixed but loose upper bound is $r_1 \leq \left(\frac{(1+\delta_m)l}{c_d f_{min} m}\right)^{1/d} = r_{max}$. Assume $l/m$ is sufficiently small so that $r_1$ is sufficiently small. By the smoothness condition, the density within $B(x,r_1)$ is lower-bounded by $f(x)-\lambda r_1$, so we have:
\begin{eqnarray}
  \mathcal{P}\left(B(x,r_1)\right) &=& (1+\delta_m)\frac{l}{m} \\
  &\geq& c_d r_1^d \left( f(x)-\lambda r_1 \right)\\
  &=& c_d r_1^d f(x)\left( 1-\frac{\lambda}{f(x)}r_1 \right) \\
  &\geq& c_d r_1^d f(x)\left( 1-\frac{\lambda}{f_{min}}r_{max} \right)
\end{eqnarray}
That is:
\begin{equation}
    r_1 \leq A(x)\left( \frac{1+\delta_m}{1-\frac{\lambda}{f_{min}}r_{max}} \right)^{1/d}
\end{equation}
Insert the expression of $r_{max}$ and set $\lambda_1 = \frac{\lambda}{f_{min}}\left(\frac{1.5}{c_d f_{min}}\right)^{1/d}$, we have:
\begin{eqnarray}
  \mathbb{E}D_{(l)}(x)-A(x) &\leq& A(x)\left( \left(\frac{1+\delta_m}{1-\lambda_1 \left(\frac{l}{m}\right)^{1/d}}\right)^{1/d} -1 \right) + C \exp\left(-\frac{\delta_m^2 l}{2(1+\delta_m)}\right) \\
  &\leq& A(x)\left( \frac{1+\delta_m}{1-\lambda_1 \left(\frac{l}{m}\right)^{1/d}}-1 \right) + C \exp\left(-\frac{\delta_m^2 l}{2(1+\delta_m)}\right) \\
  &=& A(x)\frac{\delta_m + \lambda_1 \left(\frac{l}{m}\right)^{1/d}}{1-\lambda_1\left(\frac{l}{m}\right)^{1/d}} + C \exp\left(-\frac{\delta_m^2 l}{2(1+\delta_m)}\right) \\
  &=& O\left( A(x) \lambda_1 \left(\frac{l}{m}\right)^{1/d} \right)
\end{eqnarray}
The last equality holds if we choose $l=m^{\frac{3d+8}{4d+8}}$ and $\delta_m=m^{-\frac{1}{4}}$. Similar lines follow for the lower bound. Combine these two parts and the proof is finished.

\end{proof}

\end{document}